\documentclass[11pt, reqno]{article}

\usepackage[T1]{fontenc}
\usepackage{fullpage}
\usepackage{subfig} 
\usepackage{float}
\usepackage{amsmath}
\usepackage{amsthm}
\usepackage{amssymb}
\usepackage{natbib}
\usepackage{graphicx}
\usepackage{booktabs}
\usepackage{mathtools}
\usepackage{xr,tikz}
\usepackage{enumitem}
\usepackage{nicefrac}
\usepackage{forloop}
\usepackage[hidelinks]{hyperref}
\usepackage{url}
\usepackage{bbm} 
\usepackage{authblk}
\usepackage[capitalise]{cleveref}
\usepackage{environ}
\usepackage{bm,xspace}
\usepackage{microtype}
\usepackage{mathtools}

\DeclarePairedDelimiter{\set}{\{}{\}}
\DeclarePairedDelimiter{\parens}{(}{)}

\DeclarePairedDelimiter{\floor}{\lfloor}{\rfloor}
\DeclarePairedDelimiter{\ceil}{\lceil}{\rceil}

\renewcommand{\epsilon}{\varepsilon}
\renewcommand{\emptyset}{\varnothing}
\newcommand{\ground}{V}
\newcommand{\AlgHybrid}{{\textsc{\textsc{Batch-Sieve-Streaming}\texttt{++}}}\xspace}

\newcommand{\AlgSieveStreamingPlus}{{\textsc{Sieve-Streaming}\texttt{++}}\xspace}
\newcommand{\AlgSieveStreaming}{{\textsc{Sieve-Streaming}}\xspace}
\newcommand{\AlgStreamingPreemption}{{\textsc{Preemption-Streaming}}\xspace}
\newcommand{\AlgOne}{{\textsc{Sample-One-Streaming}}\xspace}
\newcommand{\Opt}{\texttt{OPT}\xspace}
\newcommand{\LB}{\texttt{LB}\xspace}

\newcommand{\threOne}{{\texttt{Threshold}}}

\newcommand{\memory}{\texttt{B}}
\newcommand{\ratio}{\texttt{C}}
\newcommand{\buffer}{\cB}
\newcommand{\AlgSampling}{{\textsc{Threshold-Sampling}}\xspace}
\newcommand{\Prob}[2]{\Pr_{#1}\parens*{#2}}
\usepackage{amssymb}

\let\emptyset\varnothing

\DeclareMathOperator*{\argmax}{arg\,max}
\usepackage[capitalise]{cleveref}
\usepackage{eqparbox,array}
\usepackage{algorithm}
\usepackage[noend]{algorithmic}

\let\emptyset\varnothing

\newtheorem{theorem}{Theorem}
\newtheorem{lemma}{Lemma}
\newtheorem{definition}{Definition}

\usepackage[capitalise]{cleveref}
\crefname{cor}{Corollary}{Corollaries}

\usepackage{mdframed}

\makeatother

\newcommand{\defcal}[1]{\expandafter\newcommand\csname c#1\endcsname{{\mathcal{#1}}}}
\newcommand{\defbb}[1]{\expandafter\newcommand\csname b#1\endcsname{{\mathbb{#1}}}}
\newcounter{calBbCounter}
\forLoop{1}{26}{calBbCounter}{
	\edef\letter{\Alph{calBbCounter}}
	\expandafter\defcal\letter
	\expandafter\defbb\letter
}

\title{Submodular Streaming in All Its Glory: Tight Approximation, \\ Minimum Memory and 	Low 
	Adaptive Complexity}

\author[1]{Ehsan Kazemi}
 
\author[1]{Marko Mitrovic\protect}
\author[2]{Morteza Zadimoghaddam}
\author[2]{\\Silvio Lattanzi}
\author[1]{Amin Karbasi}
\affil[1]{Yale Institute for Network Science\\Yale University}
\affil[2]{Google Research, Z\"urich}
\affil[ ]{\normalsize \texttt{\{ehsan.kazemi, marko.mitrovic, amin.karbasi\}@yale.edu}}
\affil[ ]{\normalsize \texttt{\{zadim, silviol\}@google.com}}

\date{}
\begin{document}

\maketitle
\begin{abstract}
	Streaming algorithms are generally judged by the quality of their solution, memory footprint, and computational complexity. In this paper, we study the problem of maximizing a 
monotone submodular function in the streaming setting with a cardinality constraint $k$. We first propose 
\AlgSieveStreamingPlus, which requires just one pass over the data, keeps 
only $O(k)$ elements and achieves the tight $\nicefrac{1}{2}$-approximation guarantee. The best previously 
known streaming algorithms either achieve a suboptimal $\nicefrac{1}{4}$-approximation with 
$\Theta(k)$ memory or the optimal $\nicefrac{1}{2}$-approximation with $O(k\log k)$ memory. 

Next, we show that by buffering a small fraction of the stream and applying a careful filtering 
procedure, one can heavily reduce the number of adaptive computational rounds, thus substantially lowering the computational complexity of \AlgSieveStreamingPlus. We then generalize our 
results to the more challenging multi-source streaming setting. We show how one can achieve 
the tight $\nicefrac{1}{2}$-approximation guarantee with $O(k)$ shared memory while minimizing not 
only the required rounds of computations but also the total number of communicated bits. 
Finally, we demonstrate the efficiency of our algorithms on real-world data summarization tasks for multi-source streams of tweets and of YouTube videos.

\end{abstract}

\section{Introduction}\label{sec:intro}
Many important problems in machine learning, including data summarization, network inference, active set selection, facility location, and sparse regression can be cast as instances of constrained submodular maximization \citep{krause12survey}. 
Submodularity captures an intuitive diminishing returns property where the gain of adding an element to a set decreases as the set gets larger. 
More formally, a non-negative set function $f: 2^{\ground} \rightarrow \bR_{\geq 0}$ is \textbf{submodular} if for all sets $A \subseteq B \subset \ground$ and every element $e \in  \ground \setminus B$, we have $f(A \cup \{e\}) - f(A) \geq f(B \cup \{e\}) - f(B)$.
The submodular function $f$ is \textbf{monotone} if for all $A \subseteq B$ we have $f(A) \leq f(B)$ .

In this paper, we consider the following canonical optimization problem: 
given a non-negative monotone submodular function $f$, find the set $S^*$ of size at most $k$ that maximizes function $f$:
\begin{align} \label{eq:problem}
	S^* = \argmax_{S \subseteq \ground , |S| \leq k} f(S) .
\end{align}
We define $\Opt = f(S^*)$.
When the data is relatively small and it does not change over time, the greedy algorithm and other fast centralized algorithms provide near-optimal solutions.
Indeed, it is well known that for problem \eqref{eq:problem}
the greedy algorithm (which iteratively adds elements with the largest marginal gain)
achieves a $1 -\nicefrac{1}{e}$ approximation guarantee \citep{nemhauser1978analysis}.

In many real-world applications, we are dealing with massive streams of  images, videos, texts, sensor logs, tweets, and high-dimensional genomics data which are produced from different data sources.
 These data streams have an unprecedented volume and are produced so rapidly that they cannot be stored in memory, which means we cannot apply classical submodular maximization algorithms. 
 In this paper, our goal is to design efficient algorithms for streaming submodular maximization in order to simultaneously provide the best approximation factor, memory complexity, running time, and communication cost.

 For problem \cref{eq:problem}, \citet{norouzifard2018beyond} proved that any streaming algorithm\footnote{They assume the submodular function is evaluated only on feasible sets of cardinality at most $k$. In this paper, we make the same natural assumption regarding the feasible queries.} with a memory $o(n / k)$ cannot provide a solution with  an approximation guarantee better than $\nicefrac{1}{2}$.
 \AlgSieveStreaming is the first streaming algorithm with a constant approximation factor \citep{badanidiyuru2014streaming}.
 This algorithm guarantees an approximation factor of $\nicefrac{1}{2} - 
\epsilon$ and memory complexity of $O(\nicefrac{k \log(k)}{\epsilon})$.
While the approximation guarantee of their \AlgSieveStreaming  is optimal, the memory complexity is a factor of $\log(k)$ away from the desired lower bound $\Theta(k)$. 
In contrast, \mbox{\citet{buchbinder2015online}} designed a streaming algorithm with a $\nicefrac{1}{4}$-approximation factor and optimal memory $\Theta(k)$.
The first contribution of this paper is to answer the following question:
	Is there a streaming algorithm with an approximation factor arbitrarily close to $\nicefrac{1}{2}$  whose memory complexity is $O(k)$?

Our new algorithm, \AlgSieveStreamingPlus, closes the gap between the optimal approximation factor and memory complexity, but it still has some drawbacks.
In fact, in many applications of submodular maximization, the function evaluations (or equivalently Oracle queries)\footnote{The Oracle for a submodular function $f$ receives a set $S$ and returns its value $f(S)$.} are computationally expensive and  can take a long time to process.

In this context, the fundamental concept of adaptivity quantifies the number of sequential rounds required to maximize a submodular function, where in each round, we can make polynomially many independent Oracle queries in parallel.
More formally, given an Oracle $f$, an algorithm is $\ell$-adaptive if every query $\cQ$ to the Oracle $f$ at a round $1 \leq  i \leq \ell$ is independent of the answers $f(\cQ’)$ to all other queries $\cQ'$ at rounds $j$, $i \leq j \leq \ell$ \citep{balkanski2018adaptive}. 
 The adaptivity of an algorithm has important practical consequences as low adaptive complexity results in substantial speedups in parallel computing time.

All the existing streaming algorithms require at least one Oracle query for each incoming element. 
This results in an adaptive complexity of $\Omega(n)$ where $n$ is the total number of elements in the stream.
Furthermore, in many real-world applications, data streams arrive at such a fast pace that it is not possible to perform multiple Oracle queries in real time.  
This could result in missing potentially important elements or causing a huge delay.

Our idea to tackle the problem of adaptivity is to introduce a hybrid model where we allow a machine to buffer a certain amount of data, which allows us to perform many Oracle queries in parallel. We design a sampling algorithm that, in only a few adaptive rounds, picks items with good marginal gain and discards the rest.
The main benefit of this method is that we can quickly empty the buffer and continue the optimization process. 
In this way, we obtain an
algorithm with optimal approximation, query footprint, and near-optimal adaptivity.

Next, we focus on an additional challenge posed by real-world data where often multiple streams co-exist at the same time. In fact, while submodular maximization over only one stream of data is challenging, in practice there are many massive data streams generated simultaneously from a variety of sources. 
For example, these multi-source streams are generated by tweets from news agencies, videos and images from sporting events, or automated security systems and sensor logs.
These data streams have an enormous volume and are produced so rapidly that they cannot be even transferred to a central machine. 
Therefore, in the multi-source streaming setting, other than approximation factor, memory and adaptivity, it is essential to keep communication cost low.
To solve this problem, we show that a carefully-designed generalization of our proposed algorithm for single-source streams also has an optimal communication cost.

\section{Related Work}
 \citet{badanidiyuru2014streaming} were the first to  consider a one-pass streaming algorithm for maximizing a monotone submodular function under a cardinality constraint. 
\citet{buchbinder2015online} improved the memory complexity of \citep{badanidiyuru2014streaming} to $\Theta(k)$ by designing a $\nicefrac{1}{4}$ approximation algorithm.
\citet{norouzifard2018beyond} introduced an algorithm for random order streams that beats the $\nicefrac{1}{2}$ bound. 
They also studied the multi-pass streaming submodular maximization problem.
\citet{chakrabarti2015submodular} studied this problem subject to the intersection of $p$ matroid constraints. These results were further extended to more general constraints such as $p$-matchoids \citep{chekuri2015streaming, feldman2018do}.
Also, there have been some very recent works to generalize these results to non-monotone submodular functions \citep{chakrabarti2015submodular, chekuri2015streaming,chan2017online, mirzasoleiman2018streaming,feldman2018do}.
\citet{elenberg2017streaming} provide a streaming algorithm with a constant factor approximation for a  generalized notion of submodular objective functions, called weak submodularity.
In addition, a few other works study the streaming submodular maximization over  sliding windows~\citep{chen2016submodular,epasto2017submodular}

To scale to  very large datasets, several solutions to the problem of submodular maximization have been proposed  in recent years \citep{mirzasoleiman2015lazier, mirzasoleiman2016fast, feldman2017greed, badanidiyuru2014fast, mitrovic2017differentially}.
\citet{mirzasoleiman2015lazier} proposed the first linear-time algorithm for maximizing a monotone submodular function subject to a cardinality constraint that achieves a $(1 - 1/e - \epsilon)$-approximation.
\citet{buchbinder2017comparing} extended these results to non-monotone submodular functions.

Another line of work investigates the problem of scalable submodular maximization in  the MapReduce setting where  the data is split amongst several machines \citep{kumar2015fast, mirzasoleiman16distributed, barbosa2015power, mirrokni2015randomized, mirzasoleiman2016cover, barbosa2016new, liu2018submodular}.
Each machine runs a centralized algorithm on its data and passes the result to a central machine. Then, the central machine outputs the final answer. Since each machine runs a variant of the greedy algorithm, the adaptivity of all these approaches  is linear in $k$, i.e., it is $\Omega(n)$ in the worst-case.

Practical concerns of scalability have motivated studying the adaptivity of submodular maximization algorithms.
\citet{balkanski2018adaptive} showed that no algorithm can obtain a constant factor approximation in $o(\log n)$ adaptive rounds for monotone submodular maximization subject to a cardinality constraint. 
They introduced the first constant factor approximation algorithm for submodular maximization with logarithmic adaptive rounds. 
Their algorithm, in $O(\log n)$ adaptive rounds, finds a solution with an approximation arbitrarily close to $\nicefrac{1}{3}$. 
These bounds were recently improved  by $(1 - \nicefrac{1}{e} - \epsilon)$-approximation algorithm with $O(\nicefrac{\log(n)}{\text{poly}(\epsilon)})$ adaptivity
\citep{fahrbach2018submodular, balkanski2018exponential,ene2018submodular}.
More recently, \citet{chekuri2018parallelizing} studied the additivity of submodular maximization under a matroid constraint.
In addition, \citet{balkanski2018nonmonotone} proposed  an algorithm for maximizing a non-monotone submodular function with cardinality constraint $k$ whose approximation factor is arbitrarily close
to $\nicefrac{1}{(2e)}$ in $O(\log^2n)$ adaptive rounds.
\citet{fahrbach2018nonmonotone} improved the adaptive complexity of this problem to $O(\log(n))$.
\citet{chen2018unconstrained}  considered the unconstrained submodular maximization problem and proposed the first algorithm that achieves the optimal approximation guarantee in a constant number of adaptive rounds.

\paragraph{Contributions}
The main contributions of our paper are:
\begin{itemize}
	\item We introduce \AlgSieveStreamingPlus which is the first streaming algorithm with  optimal approximation factor and memory complexity.
	Note that our optimality result for the approximation factor is under the natural assumption that the Oracle is allowed to make queries only over the feasible sets of cardinality at most $k$.
	\item We design an algorithm for a hybrid model of submodular maximization, where it enjoys a near-optimal adaptive complexity and it still guarantees  both optimal approximation factor and memory complexity.
	We also prove that our algorithm has a very low communication cost in a multi-source streaming setting.
	\item We use multi-source streams of data from Twitter and YouTube to compare our algorithms against state-of-the-art streaming   approaches.
	\item We significantly improve the memory complexity for several important problems in the submodular maximization literature by applying the main idea of \AlgSieveStreamingPlus.
\end{itemize}

\section{Streaming Submodular Maximization} \label{sec:streamplus}
In this section, we propose an algorithm called \AlgSieveStreamingPlus that has the optimal $\nicefrac{1}{2}$-approximation factor and memory complexity $O(k)$.
Our algorithm is designed based on the \AlgSieveStreaming  algorithm \citep{badanidiyuru2014streaming}.

The general idea behind \AlgSieveStreaming is that choosing elements with marginal gain at least $\tau^* = \tfrac{\Opt}{2k}$ from a data stream  returns a set $S$ with an objective value of at least $f(S) \geq \frac{\Opt}{2}$.
The main problem with this primary idea is that the value of \Opt is not known. 
\citet{badanidiyuru2014streaming} pointed out that, from the submodularity of $f$, we can trivially deduce $\Delta_0 \leq \Opt \leq k \Delta_0$ where $\Delta_0$ is the largest value in the set $\{f(\{e\}) \mid e \in \ground \}$. 
It is also possible to find an accurate guess for $\Opt$ by dividing the range $[\Delta_0, k \Delta_0]$ into small intervals of $[\tau_i , \tau_{i+1})$. 
For this reason, it suffices to try $\log k$ different thresholds $\tau$ to obtain a close enough estimate of \Opt. 
Furthermore, in a streaming setting, where we do not know the maximum value of singletons a priori, \citet{badanidiyuru2014streaming} showed it suffices to only consider the range $\Delta \leq \Opt \leq 2 k\Delta $, where $\Delta$ is the maximum value of singleton elements observed so far.
The memory complexity of \AlgSieveStreaming is $O(\nicefrac{k\log k}{\epsilon})$ because there are $O(\nicefrac{\log k}{\epsilon})$ different thresholds and, for each one, we could keep at most $k$ elements.

\subsection{The \AlgSieveStreamingPlus Algorithm} \label{batch}
In the rest of this section, we show that with a novel modification to \AlgSieveStreaming it is possible to significantly reduce the memory complexity of the streaming algorithm.

Our main observation is that in the process of guessing $\Opt$, the previous algorithm uses $\Delta$ as a lower bound for $\Opt$;
but as new elements are added to sets $S_\tau$, it is possible to get better and better estimates of a lower bound on \Opt.
More specifically, we have $\Opt \geq \LB \triangleq \max_\tau f(S_\tau)$ and as a result, there is no need to keep thresholds smaller than $\frac{\LB}{2k}$.  
Also, for a threshold $\tau$ we can conclude that there is at most $\frac{\LB}{\tau}$ elements in set $S_\tau$.
These two important observations  allow us to get a geometrically decreasing upper bound on the number of items stored for each guess $\tau$, which gives the asymptotically optimal memory complexity of $O(k).$
The details of our algorithm (\AlgSieveStreamingPlus) are described in \cref{alg:Sieve-Stream++}.
Note that we represent the marginal gain of a set $A$ to the set $B$ with $f(A \mid B) = f(A \cup B) - f(B)$.
\cref{thm:sieve-stream++} guarantees the performance of \AlgSieveStreamingPlus. 
\cref{tbl:summary} compares the state-of-the-art streaming algorithms based on approximation ratio, memory complexity and queries per element.

\begin{algorithm}[htb!]
	\caption{\AlgSieveStreamingPlus}\label{alg:Sieve-Stream++}
	\vspace{0.1cm}
	\textbf{Input:} Submodular function $f$, data stream $\ground$, cardinality constraint $k$ and error term $\epsilon$
	\begin{algorithmic}[1]
		\STATE $\tau_{\min} \gets 0$, $\Delta \gets 0$ and $\LB \gets 0$ 
		\WHILE{there is an incoming item $e$ from $\ground$}
		\STATE 	$\Delta \gets \max \set{\Delta, f(\set{e})} $
		\STATE $\tau_{\min} = \frac{\max(\LB, \Delta)}{2k}$ 
		\STATE Discard all sets $S_\tau$ with $\tau < \tau_{\min}$
		\FOR{$\tau \in \{(1+\epsilon)^{i} | \nicefrac{\tau_{\min}}{(1+ 
				\epsilon)} \leq (1+\epsilon)^{i} \leq \Delta \}$}
		\STATE \textbf{if} $\tau$ is a new threshold \textbf{then}  $S_\tau \gets \emptyset$
		\IF{$|S_\tau| < k$ and $f(\{e\} \mid S_\tau) \geq \tau$} 
		\STATE $S_\tau \gets S_\tau \cup \set{e}$ and $\LB \gets \max \set{\LB, f(S_\tau)}$
		\ENDIF
		\ENDFOR
		\ENDWHILE
		\STATE \textbf{return} $\argmax_{S_\tau} f(S_\tau)$
	\end{algorithmic}
\end{algorithm}

\begin{table}
	\caption{Streaming algorithms for non-negative and monotone submodular maximization subject to a cardinality constraint $k$. 
		The \AlgSieveStreamingPlus provides the best approximation ratio, memory complexity (up to a constant factor), and query complexity.} \label{tbl:summary}
	\begin{center}
	{\small
		\begin{tabular}{lcccl}
			\toprule
			\textbf{\footnotesize Algorithm} & \textbf{\footnotesize Ratio} & \textbf{\footnotesize Memory} & \textbf{\footnotesize Queries per Element} & \textbf{\footnotesize Reference} \\
			\midrule
			\AlgStreamingPreemption & $\nicefrac{1}{4}$ & $O(k)$ & $O(k)$ & \citet{buchbinder2015online}  \\
			\AlgSieveStreaming &  $\nicefrac{1}{2} - \epsilon$ & $O(\nicefrac{k \log (k)}{\epsilon})$ & $O(\nicefrac{\log (k)}{\epsilon})$ & \citet{badanidiyuru2014streaming} \\
			\AlgSieveStreamingPlus &  $\nicefrac{1}{2} - \epsilon$ & $O(\nicefrac{k}{\epsilon})$ & $O(\nicefrac{\log (k)}{\epsilon})$ & Ours \\
			\bottomrule
		\end{tabular}
	}
	\end{center}
\end{table}

\begin{theorem} \label{thm:sieve-stream++}
	For a non-negative  monotone submodular function $f$ subject to a cardinality constraint $k$,  \AlgSieveStreamingPlus (\cref{alg:Sieve-Stream++}) returns a solution $S$ such that
	(i) $f(S) \geq (\nicefrac{1}{2} - \epsilon) \cdot \max_{A \subseteq \ground, |A| \leq k} f(A)$,
	(ii) memory complexity is $O(\nicefrac{k}{\epsilon})$,
	and (iii) number of queries is $O(\nicefrac{\log (k)}{\epsilon})$ per each element.
\end{theorem}
\begin{proof}
	\noindent
	\textbf{Approximation guarantee}
	The approximation ratio is proven very similar to the approximation guarantee of \AlgSieveStreaming \cite{badanidiyuru2014streaming}.
	Let's define $S^{*} = \argmax_{A \subseteq \ground, |A| \leq k} f(A)$, $\Opt = f(S^{*})$ and $\tau^* = \frac{\Opt}{2k}$.
	We further define $\Delta = \max_{e \in \ground} f(\set{e})$. 
	It is easy to observe that $\max \set{\Delta, \LB} \leq \Opt \leq k \Delta$ and there is a threshold $\tau$ such that $(1 - \epsilon)\tau^* \leq \tau < \tau^*$. 
	Now consider the set $S_\tau$.
	\AlgSieveStreamingPlus adds elements with a marginal gain  at least $\tau$ to the set $S_{\tau}$. We have two cases:
	\begin{itemize}
		\item $|S_{\tau}| = k:$ We define $S_{\tau} = \set{e_1, \cdots, e_k}$ where $e_i$ is the $i$-th picked element. Furthermore, we define $S_{\tau,i} = \set{e_1, \cdots, e_i}$. We have
		\[ f(S_{\tau})  = \sum_{i=1}^{k}  f(\{e_i\} \mid S_{\tau, i-1}) \geq k \tau \geq (\frac{1}{2} - \epsilon)  \cdot \Opt.\]
		This is true because the marginal gain of each element at the time it has been added to the set $S_\tau$ is at least $\tau^{*}$.
		\item $|S_{\tau}| < k:$ We have
		\begin{align*}
		\Opt  \leq  f(S^* \cup S_{\tau}) & = f(S_{\tau}) + f(S^* \mid S_{\tau}) \\
	 & 	\overset{(a)}{\leq}  f({S_\tau}) + \sum_{e \in S^* \setminus S_{\tau}} f(\{e\} \mid S_{\tau})  \\
	 & \overset{(b)}{\leq} f({S_\tau}) + k \tau^* \\
	 & = 
	 f({S_\tau}) +  \dfrac{\Opt}{2},
		\end{align*}
		where $(a)$ is correct because $f$ is a submodular function, and we have $(b)$ because each element of $S^*$ that is not picked by the algorithm has had a marginal gain of less than $\tau < \tau^*$.
	\end{itemize}
	
	\noindent
	\textbf{Memory complexity}
	Let $S_\tau$ be the set we maintain for threshold $\tau$. We know that $\Opt$ is at least $\LB = \max_\tau f(S_\tau) \geq  \max_\tau (|S_\tau| \times \tau)$ because the marginal gain of an element in set $S_\tau$ is at least $\tau$. 
	Note that $\LB$ is the best solution found so far.
	Given this lower bound on $\Opt$ (which is potentially better than $\Delta$ if we have picked enough items), we can dismiss all thresholds that are too small, i.e., remove all thresholds $\tau < \frac{\LB}{2k} \leq \frac{\Opt}{2k}$.
	For any remaining $\tau \geq \frac{\LB}{2k}$, we know that $|S_\tau|$ is at most $\frac{\LB}{\tau}$.
	We consider two sets of thresholds: (i) $ \frac{\LB}{2k} \leq  \tau \leq  \frac{\LB}{k}$, and (ii) $\tau \geq \frac{\LB}{k}$.
	For the first group of thresholds, the bound on $|S_\tau|$ is the trivial upper bound of $k$. 
	Note that we have $\log_{1+\epsilon} (2) \leq \lceil \log(2) / \epsilon \rceil$ of such thresholds. 
	For the second group of thresholds, as we increase $\tau$, for a fixed value of $\LB$ the upper bound on the size of $S_\tau$ gets smaller.
	Indeed, these upper bounds are geometrically decreasing values with the first term equal to $k$. And they reduce by a coefficient of $(1 + \epsilon)$ as thresholds increase by a factor of $(1+\epsilon)$. 
	Therefore, we can bound the memory complexity by
	\[\textrm{Memory complexity} \leq  \left\lceil \dfrac{k \log(2)}{\epsilon} \right\rceil  + \sum_{i = 0}^{\log_{1+\epsilon}(k)} \dfrac{k}{(1+\epsilon)^i} = O\left( \dfrac{k}{\epsilon} \right). \]
	Therefore, the total memory complexity is $O(\frac{k}{\epsilon})$.
	
	\noindent
	\textbf{Query complexity} For every incoming element $e$, in the worst case, we should compute the marginal gain of $e$ to all the existing sets $S_\tau$. Because there is $O(\frac{\log k}{\epsilon})$ of such sets (the number of different thresholds), therefore the query complexity per element is $O(\frac{\log k}{\epsilon})$.
\end{proof}

\subsection{The \AlgHybrid Algorithm}
\label{sec:single-alg}

The \AlgSieveStreamingPlus algorithm, for each incoming element of the stream, requires at least one query to the Oracle which increases its adaptive complexity to $\Omega(n)$.
Since the adaptivity of an algorithm has a significant impact on its ability to be executed in parallel, there is a dire need to implement streaming algorithms with low adaptivity.
To address this concern, our proposal is to first buffer a fraction of the data stream and then, through a parallel threshold filtering procedure, reduce the adaptive complexity, thus substantially lower the running time.
Our results show that  a small buffer memory can significantly parallelize streaming submodular maximization.

One natural idea to parallelize the process of maximization over a buffer is to iteratively perform the following two steps:
(i) for a threshold $\tau$, in one adaptive round, compute the marginal gain of elements to set $S_\tau$ and discard those with a gain less than $\tau$, and
(ii) pick one of the remaining items with a good marginal gain and add it to $S_\tau$. 
This process is repeated at most $k$ times.
We refer to this algorithm as \AlgOne and we will use it as a baseline in \cref{multiSourceExperiments}.

Although by using this method we can find a solution with $\nicefrac{1}{2} - \epsilon$ approximation factor, the adaptive complexity of this algorithm is $\Omega(k)$ which is still prohibitive in the worst case.
For this reason, we introduce a hybrid algorithm called \AlgHybrid.
This algorithm enjoys two important properties: (i) the number of adaptive rounds is near-optimal, and (ii) it has an optimal memory complexity (by adopting an idea similar to \AlgSieveStreamingPlus).
Next, we explain \AlgHybrid (\cref{alg:hybrid}) in detail.

First, we assume that the machine has a buffer $\buffer$ that can store at most $\memory$ elements.
For a data stream $\ground,$ whenever $\threOne$ fraction of the buffer is full, the optimization process begins.
The purpose of $\threOne$ is to empty the buffer before it gets completely full and to avoid losing arriving elements.
Similar to the other sieve streaming methods, \AlgHybrid requires us to guess the value of $\tau^* = \frac{\Opt}{2k}$.
For each guess $\tau$,  \AlgHybrid uses \AlgSampling (\cref{alg:Sampling}) as a subroutine.
\AlgSampling iteratively picks random batches of elements $T$.
If their average marginal gain to the set of picked elements $S_{\tau}$ is at least $(1-\epsilon)\tau$ it adds that batch to $S_{\tau}$. Otherwise, all elements with marginal gain less than $\tau$ to the set $S_{\tau}$ are filtered out.
\AlgSampling repeats this process until the buffer is empty or $|S_\tau|=k$.

Note that in \cref{alg:hybrid}, we define the function $f_S$ as $f_S(A) = f(A \mid S) $, which calculates the marginal gain of adding a set $A$ to $S$. 
It is straightforward to show that if $f$ is a non-negative and monotone submodular function, then $f_S$ is also non-negative and monotone submodular.

The adaptive complexity of \AlgHybrid is the number of times its buffer gets full (which is at most $N/\memory$) multiplied by the adaptive complexity of \AlgSampling. 
The reason for the low adaptive complexity of \AlgSampling is quite subtle.
In Line~\ref{line:filter} of \cref{alg:Sampling}, with a non-negligible probability, a constant fraction of items is discarded from the buffer. 
Thus, the while loop continues for at most $O(\log \memory)$  steps. 
Since we increase the batch size by a constant factor of $(1+\epsilon)$ each time, the for loops within each while loop will run at most $O(\log (k) / \epsilon)$ times. 
Therefore, the total adaptive complexity of \AlgHybrid is $O(\frac{N \log (\memory) \log (k)}{\memory  \epsilon} )$
Note that when $|S| < 1/\epsilon$, multiplying the size by $(1+\epsilon)$ would increase it less than one, so we increase the batch size one by one for the first loop in Lines~\ref{line:loop1:begin}--\ref{line:loop1:end} of \cref{alg:Sampling}.
\cref{thm:batch-streaming} guarantees the performance of \AlgHybrid.

\begin{algorithm}[htb!]
	\caption{\AlgHybrid}\label{alg:hybrid}
	\vspace{0.1cm}
	\textbf{Input:} Stream of data $\ground,$ submodular set function $f$, cardinality constraint $k$, buffer $\buffer$ with a memory $\memory$,  $\threOne$, and error term $\epsilon$.
	\begin{algorithmic}[1]
		\STATE $\Delta \gets 0, \tau_{\min} \gets 0, \LB \gets 0$ and $\buffer \gets \emptyset$
		\WHILE{there is an incoming element $e$ from $\ground$}
		\STATE Add $e$ to $\buffer$
		\IF{the buffer $\buffer$ is $\threOne$ percent full}
		\STATE 	$\Delta \gets\max\{ \Delta,\max_{e \in \buffer} f(e)  \} $, $\tau_{\min} = \frac{\max(\LB, \Delta)}{2k(1+\epsilon)}$
		\STATE Discard all sets $S_\tau$ with $\tau < \tau_{\min}$
		\FOR{$\tau \in \{(1+\epsilon)^{i} | \tau_{\min} \leq (1+\epsilon)^{i} \leq \Delta \}$}  
		\STATE If $\tau$ is a new threshold then assign a new set $S_\tau$ to it, i.e., $S_\tau \gets \emptyset$
		\IF{$|S_\tau| < k$}
		\STATE $T \gets \AlgSampling(f_{S_{\tau}},\buffer, k - |S_\tau|, \tau, \epsilon)$
		\STATE $S_\tau \gets S_\tau \cup T$
		\ENDIF
		\ENDFOR
		\STATE $\LB = \max_{S_\tau} f(S_\tau)$ and $\buffer \gets \emptyset$
		\ENDIF
		\ENDWHILE
		\STATE \textbf{return} $\argmax_{S_\tau} f(S_\tau)$
	\end{algorithmic}
\end{algorithm}

\begin{algorithm}[H]
	\caption{\AlgSampling }\label{alg:Sampling}
	\vspace{0.1cm}
	\textbf{Input:} Submodular set function $f$,  set of buffered items $\buffer$,	
cardinality constraint $k$, threshold $\tau,$ and error term $\epsilon$
\begin{algorithmic}[1]
		\STATE  $S  \gets \emptyset$
		\WHILE{$ |\buffer| > 0$ and $|S| < k$ \label{line:filter-loop}}
		\STATE update $\buffer \gets \set{x \in \buffer : f(\{x\} \mid S) \geq \tau}$  and filter out the rest \label{line:filter}
		\FOR{$i=1$ to $\ceil{\frac{1}{\epsilon}}$ \label{line:loop1:begin}} 
		\STATE Sample $x$ uniformly at random from $\buffer \setminus S$  \label{line:sample-one}
		\IF{$f(\set{x} | S) \leq (1 - \epsilon) \tau$}
		\STATE \textbf{break} and go to Line~\ref{line:filter-loop} \label{line:break-one}
		\ELSE
		\STATE $S \gets S \cup \set{x}$ \label{line:pick-one}
		\STATE \textbf{if} $|S| = k$ \textbf{then return} $S$ \label{line:loop1:end}
		\ENDIF
		\ENDFOR
		\FOR{$i=\floor{\log_{1+\epsilon}(\nicefrac{1}{\epsilon})}$ to $ \ceil{\log_{1+\epsilon} k} - 1$ \label{line:loop2:begin}}
		\STATE $t \gets \min \set{ \floor{(1 + \epsilon)^{i+1} - (1 + \epsilon)^{i}},  |\buffer \setminus S|, k - |S|}$
		\STATE Sample a random set $T$ of size $t$ from $\buffer \setminus S$  \label{line:pick-batch}
		\IF{$|S \cup T| = k$} 
		\STATE \textbf{return} $S \cup T$ \label{line:check-size}
		\ENDIF
		\IF{$\dfrac{f(T \mid S)}{|T|} \leq (1 - \epsilon) \tau$ \label{line:test-margin} }
		\STATE$S \gets S \cup T$ and \textbf{break}  \label{line:break-insertion}
		\ELSE
		\STATE $S \gets S \cup T$  \label{line:end-loop}\label{line:loop2:end}
		\ENDIF
		\ENDFOR
		\ENDWHILE
		\STATE \textbf{return} $S$
	\end{algorithmic}
\end{algorithm}

\newpage

\begin{theorem} \label{thm:batch-streaming}
	For a non-negative monotone submodular function $f$ subject to a cardinality constraint $k$, 
define $N$ to be the total number of elements in the stream, $\memory$ to be the buffer size and  $\epsilon < \nicefrac{1}{3}$ to be a constant.
For  \AlgHybrid we have: 
(i) the approximation factor is $\nicefrac{1}{2} - \nicefrac{3 \epsilon}{2}$, (ii) the memory complexity is $O( \memory + \nicefrac{k}{\epsilon})$, and (iii) the expected adaptive complexity is $O(\frac{N \log (\memory) \log (k)}{\memory  \epsilon} )$.
\end{theorem}

\begin{proof}
	\textbf{Approximation Guarantee:} 
	Assume $\buffer$ is the set of elements buffered from  the stream $\ground$. 
	Let's define $S^{*} = \argmax_{A \subseteq \ground, |A| \leq k} f(A)$, $\Opt = f(S^{*})$ and $\tau^* = \frac{\Opt}{2k}$.
	Similar to the proof of \cref{thm:sieve-stream++}, we can show that \AlgHybrid considers a range of thresholds such that for one of them (say $\tau$) we have $\frac{\Opt(1-\epsilon)}{2k} \leq \tau < \frac{\Opt}{2k}$.
	In the rest of this proof, we focus on $\tau$ and its corresponding set of picked items $S_\tau$. 
	For set $S_\tau$ we have two cases:
	\begin{itemize}
		\item if $|S_{\tau}| < k$, we have:
		\begin{align*}
		\Opt  \leq f(S^* \cup S_{\tau}) & = f(S_{\tau}) + f(S^* \mid S_{\tau})  \\
	&	\overset{(a)}{\leq}  f({S_\tau}) + \sum_{e \in S^* \setminus S_{\tau}} f(\{e\} \mid S_{\tau}) \\
	&\overset{(b)}{\leq} f({S_\tau}) + k \tau^* \\
	&	= f({S_\tau}) +  \dfrac{\Opt}{2},
		\end{align*}
		where inequality $(a)$ is correct because $f$ is a submodular function.  And inequality $(b)$ is correct because each element of the optimal set $S^*$ that is not picked by the algorithm, i.e., it had discarded in the filtering process, has had a marginal gain of less than $\tau < \tau^*$.
		
		\item if $|S_{\tau}| = k$: 
		Assume the set $S_{\tau}$ of size $k$ is sampled in $\ell$ iterations of the while loop in Lines \ref{line:filter-loop}--\ref{line:end-loop} of \cref{alg:Sampling}, and $T_i$ is the union of sampled batches  in the $i$-th iteration of the while loop. 
		Furthermore, let $T_{i,j}$ denote the $j$-th sampled batch in the $i$-th iteration of the while loop.
		We define $S_{\tau, i, j} = \bigcup_{i, j} T_{i,j}$, i.e, $S_{\tau, i, j}$ is the state of set $S_\tau$ after the $j$-th batch insertion in the $i$-th iteration of the while loop.
		We first prove that the average gain of each one of these sets $T_i$ to the set $S_{\tau}$ is at least $(1-2\epsilon) \cdot |T_i|$.
		
		To lower bound the average marginal gain of $T_i$, for each $T_{i,j}$ we consider three different cases: 
		\begin{itemize}
			\item  the while loop breaks at Line~\ref{line:break-one} of \cref{alg:Sampling}: 
			We know that the size of all $T_{i,j}$ is one in this case.
			It is obvious that $ f(T_{i,j}\mid S_{\tau, i, j - 1}) \geq (1 - \epsilon) \cdot \tau.$
			
			\item   \AlgSampling passes the first loop and does not break at Line~\ref{line:break-insertion}, i.e., it continues to pick items till $S_{\tau} = k$ or  the buffer memory is empty:  again it is obvious that \[ f(T_{i,j} \mid S_{\tau, i, j-1}) \geq (1 - \epsilon) \cdot \tau \cdot |T_{i,j}|\]
			This is true because when set $T_{i,j}$ is picked, it has passed the test at Line~\ref{line:test-margin}.
			Note that it is possible the algorithm breaks at Line~\ref{line:check-size} without passing the test at Line~\ref{line:test-margin}.
			If the average marginal gain of the sampled set $T_{i,j}$ is more than $(1 - \epsilon) \cdot \tau$ then the analysis would be exactly the same as the current case. Otherwise, we handle it similar to the next case where the sampled batch does not provide the required average marginal gain.
			
			\item  it passes the first loop and breaks at Line~\ref{line:break-insertion}. We have the two following observations:
			\begin{enumerate}
				\item in the current while loop, from the above-mentioned cases, we conclude that the average marginal gain of all the element picked before the last sampling is at least $(1-\epsilon) \cdot \tau$, i.e.,
				\[\forall r, 1 \leq r < j : f(T_{i,r} \mid S_{\tau, i, r-1}) \geq (1 - \epsilon) \cdot \tau \cdot |T_{i,r}|.\]
				\item the number of elements which are picked at the latest iteration of the while loop is at most $\epsilon$ fraction of all the elements picked so far (in the current while loop), i.e., $|T_{i,j}| \leq \epsilon \cdot |\bigcup_{1 \leq r < j} T_{i,r}|$ and $|T_i| \leq (1+ \epsilon) \cdot |\bigcup_{1 \leq r < j} T_{i,r}|$. 
				Therefore, from the monotonicity of $f$, we have
				\begin{align*}
				f(\bigcup_{1 \leq r \leq  j} T_{i,r} \mid S) 
			 &	\geq f(\bigcup_{1 \leq r < j} T_{i,r} \mid S)  \\
			&	\geq  (1 - \epsilon) \cdot \tau \cdot  \mid \bigcup_{1 \leq r < j} T_{i,r} | \\
			&	\geq \dfrac{|T_i| \cdot \tau \cdot (1 - \epsilon)}{(1+\epsilon)}  \\
			&	\geq (1 - 2\epsilon) \cdot \tau \cdot |T_i|.
					\end{align*}
			\end{enumerate}
		\end{itemize}
		To sum-up, we have
		\begin{align*}
		f(S_{\tau})  = \sum_{i=1}^{\ell}  f(T_i \mid S_{\tau, i-1}) \geq \sum_{i=1}^{\ell} (1 - 2\epsilon) \cdot \tau \cdot |T_i|  = (1 - 2\epsilon)\cdot \tau \cdot k \geq  \left(\frac{1}{2} - \frac{3 \epsilon}{2} \right) \cdot \Opt.
		\end{align*}
	\end{itemize}

	\textbf{Memory complexity} In a way similar to analyzing the memory complexity of \AlgSieveStreamingPlus,  we conclude that the required memory of \AlgHybrid in order to store solutions for different thresholds is also $O(\frac{k}{\epsilon})$. Since we buffer at most $\memory$ items, the total memory complexity is $O(\memory + \frac{k}{\epsilon})$.

	\paragraph{Adaptivity Complexity of \AlgSampling}
	
	To guarantee the adaptive complexity of our algorithm, we first upper bound the expected number of iterations of the while loop in Lines \ref{line:filter-loop}--\ref{line:end-loop} of \cref{alg:Sampling}.
	\begin{lemma}\label{lemma:filtering-rounds}
		For any constant $\epsilon > 0$, the expected number of iterations in the while loop of Lines \ref{line:filter-loop}--\ref{line:end-loop} of \AlgSampling is $O(\log(|\buffer|))$ where $\buffer$ is the set of buffered elements passed to \AlgSampling. 
	\end{lemma}

	We defer the proof of \cref{lemma:filtering-rounds} to \cref{sec:lemma-proof}. Next, we discuss how this lemma translates to the total expected adaptivity of \AlgHybrid.
	
	There are at most $\ceil{\frac{1}{\epsilon}} + \log_{1+\epsilon}k =O(\frac{\log k}{\epsilon})$ adaptive rounds in each iteration of the while loop of \cref{alg:Sampling}. 
	So, from \cref{lemma:filtering-rounds} we conclude that the expected adaptive complexity of each call to \AlgSampling is $O(\frac{\log(|\buffer|)\log (k)}{\epsilon})$.
	To sum up, the general adaptivity of the algorithm takes its maximum  value when the number of times a buffer gets full is the most, i.e., when $|\buffer| = \threOne \times \memory$ for $\frac{N}{\threOne \times \memory}$ times.
	We assume $\threOne$ is constant. Therefore, the expected adaptive complexity is $O(\frac{N \log (\memory) \log (k)}{\memory \epsilon} )$.
\end{proof}

\noindent \textbf{Remark}
It is important to note that \AlgSampling is inspired by recent  progress for maximizing submodular functions with low adaptivity  \citep{fahrbach2018submodular, balkanski2018exponential, ene2018submodular} but it uses a few new ideas to adapt the result to our setting. 
Indeed, if we had used the sampling routines from these previous works, it was even possible to slightly improve the adaptivity of the hybrid model. 
The main issue with these methods is that their adaptivity heavily depends on evaluating many random subsets of the ground set in each round.
As it is discussed in the next section, we are interested in algorithms that are efficient in  the multi-source setting. 
In that scenario, the data is distributed among several machines, so existing sampling methods dramatically increases the communication cost of our hybrid algorithm.

\section{Multi-Source Data Streams} \label{sec:multisource-alg}

In general, the important aspects to consider for a single source streaming algorithm are approximation factor, memory complexity, and adaptivity.
In the multi-source setting, the communication cost of an algorithm also plays an important role. 
While the main ideas of \AlgSieveStreamingPlus  give us an optimal approximation factor and memory complexity, there is always a trade-off  between adaptive complexity and communication cost in any threshold sampling procedure.

As we discussed before, existing submodular maximization algorithms with low adaptivity need to evaluate the utility of random subsets several times to guarantee the marginal gain of sampled items. 
Consequently, this incurs high communication cost.
In this section, we explain how \AlgHybrid can be generalized to the multi-source scenario with both low adaptivity and low communication cost.

We assume elements arrive from $m$ different data streams and for each stream the elements are placed in a separate machine with a buffer $\buffer_i$. 
When the buffer memory of at least one of these $m$ machines is $\threOne\%$ full, the process of batch insertion and filtering begins.
The only necessary change to $\AlgHybrid$ is to use a parallelized version of \AlgSampling with inputs from $\{ \buffer_{i}\}$.
In this generalization, Lines \ref{line:sample-one} and \ref{line:pick-batch} of \cref{alg:Sampling} are executed in a distributed way where the goal is to perform the random sampling procedure from the buffer memory of all machines.
Indeed, in order to pick a batch of $t$ random items, the central coordinator asks each machine to send a pre-decided number of items.
Note that the set of picked elements $S_\tau$ for each threshold $\tau$ is shared among all machines.
And therefore the filtering step at Line~\ref{line:filter} of \cref{alg:Sampling} can be done independently for each stream in only one adaptive round.
Our algorithm is shown pictorially in Figure \ref{fig:schematic}.
\cref{thm:batch-multi-streaming} guarantees the communication cost of \AlgHybrid in the multi-source setting. Notably, the communication cost of our algorithm is independent of the buffer size $\memory$ and the total number of elements $N$.

\begin{figure}[htb!] 
	\centering     
\includegraphics[width=110mm]{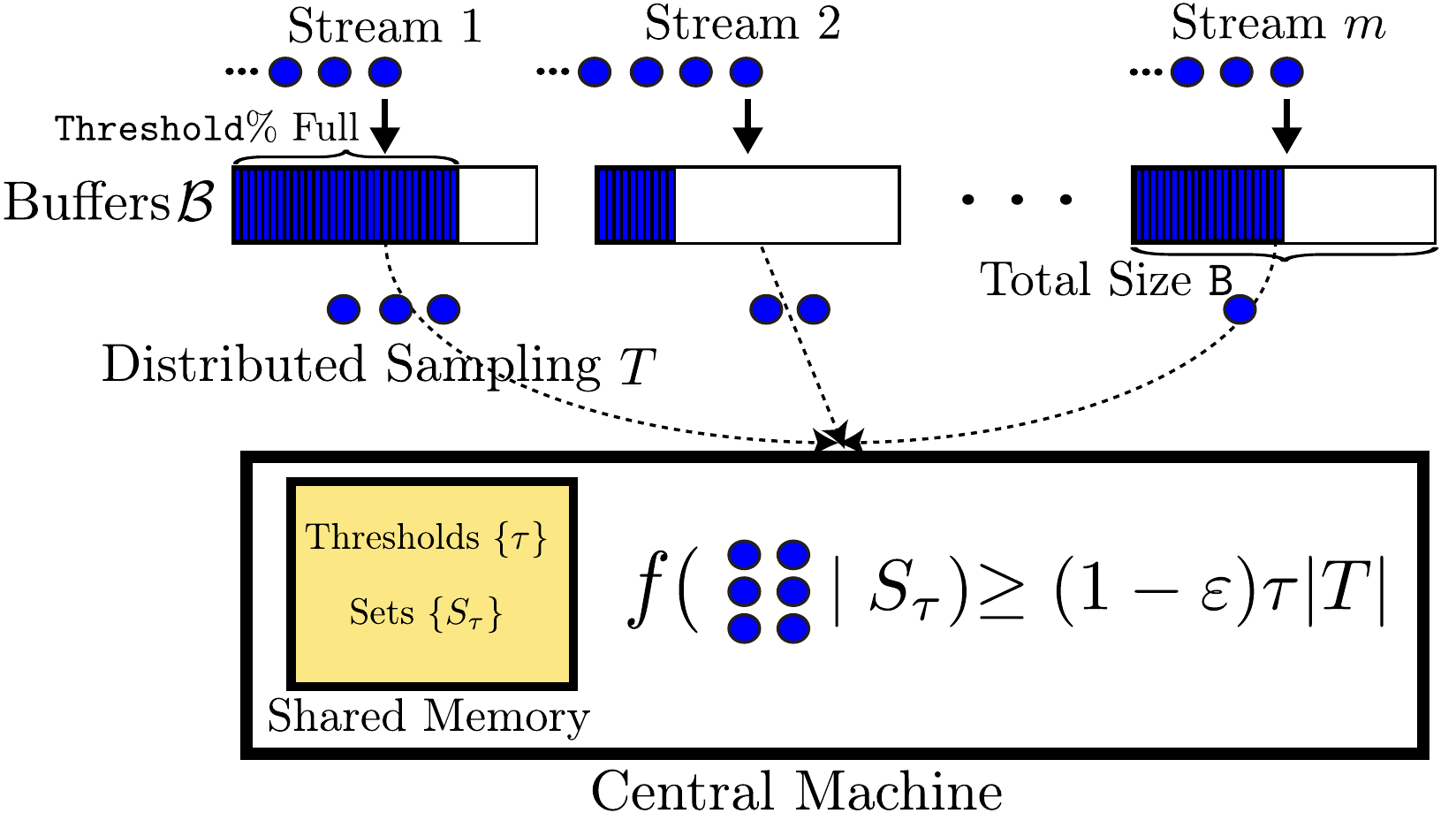}
\caption{The schematic representation of our proposed hybrid algorithm: 
	there are $m$ simultaneous streams where data from each stream is buffered separately. 
	When a buffer is \threOne\% full, a central machine starts the sampling process.
	The thresholds $\{\tau\}$ and sets $\{S_\tau \}$ are stored in a shared memory.
	First, for each threshold $\tau$, all elements with marginal gain less than $\tau$ are discarded from the buffers. Then the central machine randomly samples $t$ items $T$ (with geometrically increasing values of $t$) from the buffers of all streams and adds them to set $S_\tau$ if their average marginal gain is at least $(1-\epsilon)\tau$. 
The sampling procedure stops when the average value of randomly picked items is not good enough.
These iterative steps are performed until $k$ items are picked or the buffer memories of all machines are emptied.\label{fig:schematic}
}
\end{figure}

\begin{theorem} \label{thm:batch-multi-streaming}
	For a non-negative and monotone submodular function $f$ in a multi-source streaming setting subject to a cardinality constraint $k$, define $\Delta_0$ as the largest singleton value when for the first time a buffer gets full, and $\ratio = \frac{{\textup{\Opt}}}{\Delta_{0}}$.
	The total communication cost of \AlgHybrid is $O(\frac{k \log \ratio}{\epsilon^2})$.
\end{theorem}

\begin{proof}
For $m$ different data streams, assume $\buffer_{i}$ is the set of elements buffered from  the $i$-th stream. 
We define $\buffer$ to be the union of all elements from all streams.
The communication cost of \AlgHybrid in the multi-source setting is the total number of elements sampled (in a distributed way) from all sets $\{\buffer_{i}\}_{1 \leq i \leq m}$ in Lines~\ref{line:sample-one} and \ref{line:pick-batch} of \cref{alg:Sampling}.

As a result, we can conclude that the communication cost is at most twice the number of elements  has been in a set $S_\tau$ at a time during the run of the algorithm.
To see the reason for this argument, note that because the filtering step happens just before the for loop of Lines~\ref{line:loop1:begin}--\ref{line:loop1:end}, the first picked sample in this for loop always passes the test and is added to $S_\tau$. 
Furthermore, all the items sampled  at Line~\ref{line:pick-batch}, irrespective of their marginal gain, are added to $S_\tau$.
So, in the worst case scenario, the communication complexity is maximum when
the for loop breaks always
at the second instance of the sampling process of Line~\ref{line:sample-one} (after one successful try).
Therefore, we only need to upper bound the total number of elements which at some point has been in a set $S_\tau$ at one of the calls to \AlgSampling.

The first group of thresholds  the \AlgHybrid algorithm considers  the interval $[\nicefrac{\Delta_{0}}{(2k)}, \Delta_{0}]$, where in the beginning we have $\LB = \Delta_{0}$.
Following the same arguments as the proof of \cref{thm:sieve-stream++}, we can show that if neither $\LB$ nor $\Delta$ changes, the total number of elements in sets $\{ S_\tau \}$ is $O(\frac{k}{\epsilon})$.
We define $\Delta_{\max}$ to be the largest singleton element in the whole data streams.
The number of times the interval of thresholds changes because of the change in $\Delta$ is $\log_{1+\epsilon}( \nicefrac{\Delta_{\max}}{\Delta_0} )$.
Furthermore, by changes in $\LB$ some thresholds and their corresponding sets are deleted and new elements might be added.
The number of times $\LB$ changes is upper bounded by $\log_{1+\epsilon}( \nicefrac{\Opt}{\Delta_0} )$.
Note that we have $\Delta_{\max} \leq \Opt$.
From the fact that the number of changes in the set of thresholds is upper bounded by $\log_{1+\epsilon}( \nicefrac{\Delta_{\max}}{\Delta_0})  = O(\frac{\log \ratio}{\epsilon})$ and the number of elements in $\{S_\tau\}$ at every step of the algorithm is $O(\frac{k}{e})$,
we conclude the total communication cost of \AlgHybrid is 
$O(\frac{k \log \ratio}{\epsilon^2})$.
\end{proof}
\section{Experiments}

In these experiments, we have three main goals:
\begin{enumerate}
\item For the single-source streaming scenario, we want to demonstrate the memory efficiency of \AlgSieveStreamingPlus relative to \AlgSieveStreaming.
\item For the multi-source setting, we want to showcase how \AlgHybrid requires the fewest adaptive rounds amongst algorithms with optimal communication costs.
\item Lastly, we want to illustrate how a simple variation of \AlgHybrid can trade off communication cost for adaptivity, thus allowing the user to find the best balance for their particular problem.
\end{enumerate}

\subsection{Datasets}

These experiments will be run on a Twitter stream summarization task and a YouTube Video summarization task, as described next. 

\paragraph{Twitter Stream Summarization} \label{sec:twitter}
In this application, we want to produce real-time summaries for Twitter feeds. As of January 2019, six of the top fifty Twitter accounts (also known as ``handles'') are dedicated primarily to news reporting. Each of these handles has over thirty million followers, and there are many other news handles with tens of millions of followers as well. Naturally, such accounts commonly share the same stories. Whether we want to provide a periodic synopsis of major events or simply to reduce the clutter in a user's feed, it would be very valuable if we could produce a succinct summary that still relays all the important information. 

To collect the data, we scraped recent tweets from 30 different popular news accounts, giving us a total of 42,104 unique tweets. In the multi-source experiments, we assume that each machine is scraping one page of tweets, so we have 30 different streams to consider. 

We want to define a submodular function that covers the important stories of the day without redundancy. 
To this end, we extract the keywords from each tweet and weight them proportionally to the number of retweets the post received.
 In order to encourage diversity in a selected set of tweets, we take the square root of the value assigned to each keyword. 
More formally, consider a function $f$ defined over a ground set $\ground$ of tweets.
Each tweet $e \in \ground$ consists of a positive value $\text{val}_e$ denoting its number of retweets and a set of $\ell_e$ keywords $W_e = \{ w_{e,1}, \cdots, w_{e, \ell_e}\}$ from a general set of keywords $\cW$.
The score of a word $w \in W_e $ for a tweet $e$ is defined by $\text{score}(w,e) = \text{val}_e$.
If $w \notin W_e $, we  define $\text{score}(w,e) = 0$.
For a set $S \subseteq V$ of tweets, the function $f$ is defined as follows:
\begin{align*} \label{eq:function-twitter}
	f(S) = \sum_{w \in \cW} \sqrt{\sum_{e \in S} \text{score}(w,e)}.
\end{align*}
Figure \ref{tweetExample} demonstrates how we calculate the utility of a set of tweets.
 In \cref{sec:twitter-function}, we give proof of submodularity of this function.

\begin{figure}[htb!] 
	\centering     
	\fbox{\includegraphics[width=74mm]{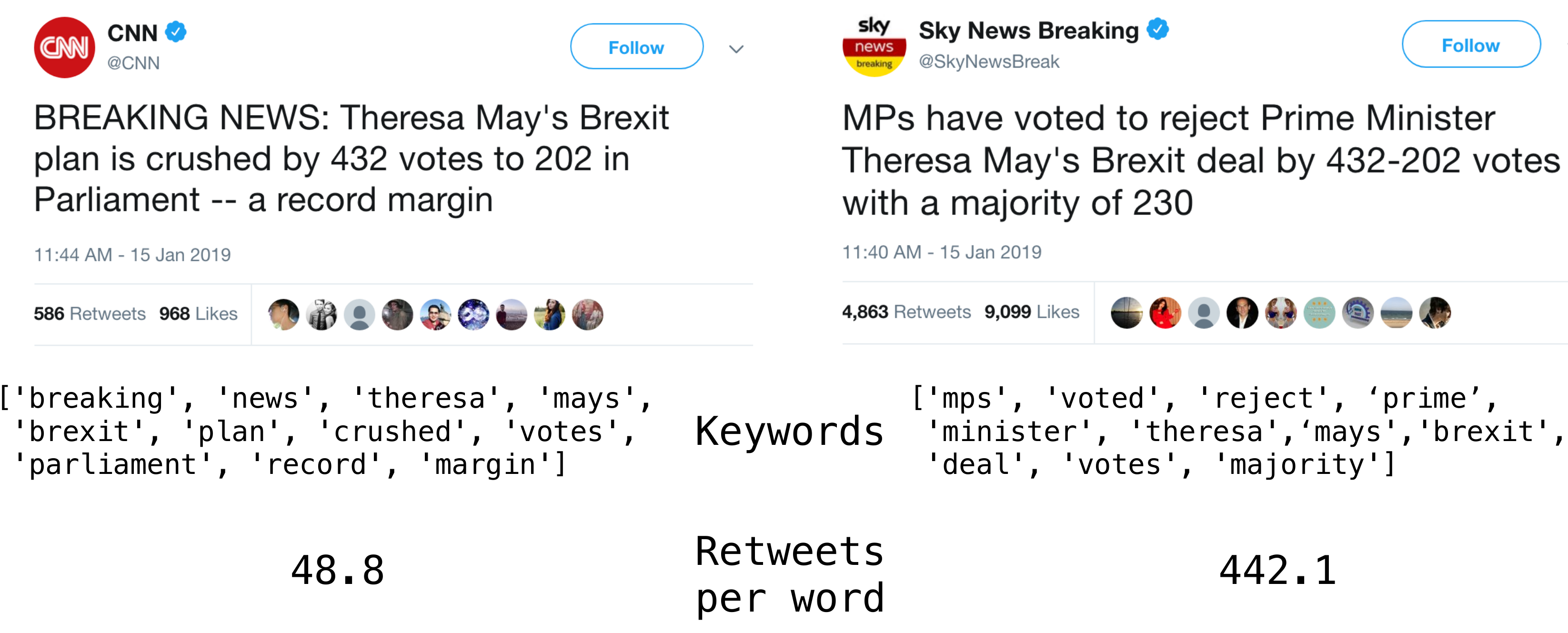}} \vspace{0.1in}
	\quad
	\fbox{\includegraphics[width=79mm]{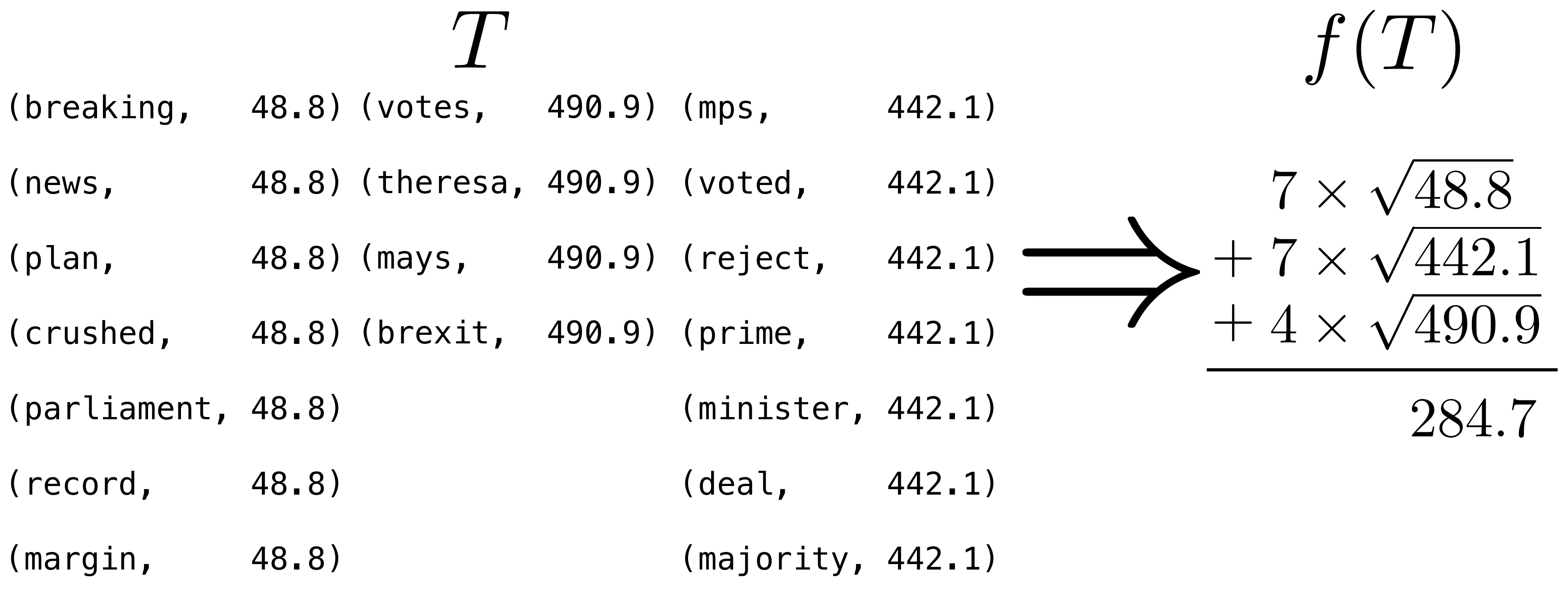}}
	\caption{At the top, we show two tweets on the same subject from different accounts. We first extract the list of keywords, as well as the number of retweets per word. We combine these into a single list $T$ of (keyword, score) pairs and then pass this list through our submodular function $f$.}
	\label{tweetExample}
\end{figure}

\paragraph{YouTube Video Summarization} \label{sec:youtube}

In this second task, we want to select representative frames from multiple simultaneous and related video feeds. In particular, we consider YouTube videos of New Year's Eve celebrations from ten different cities around the world. Although the cities are not all in the same time zone, in our multi-source experiments we assume that we have one machine processing each video simultaneously.

Using the first 30 seconds of each video, we train an autoencoder that compresses each frame into a 4-dimensional representative vector. Given a ground set $\ground$ of such vectors, we define a matrix $M$ such that $M_{ij} = e^{-\text{dist}(v_i,v_j)}$, where $\text{dist}(v_i,v_j)$ is the euclidean distance between vectors $v_i, v_j \in \ground$. Intuitively, $M_{ij}$ encodes the similarity between the frames represented by  $v_i$ and $v_j$.

The utility of a set $S \subseteq \ground$ is defined as a  non-negative monotone submodular objective $f(S) = \log \det( I + \alpha M_S )$, where $I$ is the identity matrix, $\alpha > 0$ and $M_S$ is the principal sub-matrix of $M$ indexed by $S$ \citep{herbrich2003fast}. Informally, this function is meant to measure the diversity of the vectors in $S$. 
Figure \ref{sampleFrames} shows the representative images selected by our \AlgHybrid algorithm for $k= 8$.

\begin{figure}[htb!] 
	\centering     
\includegraphics[width=140mm]{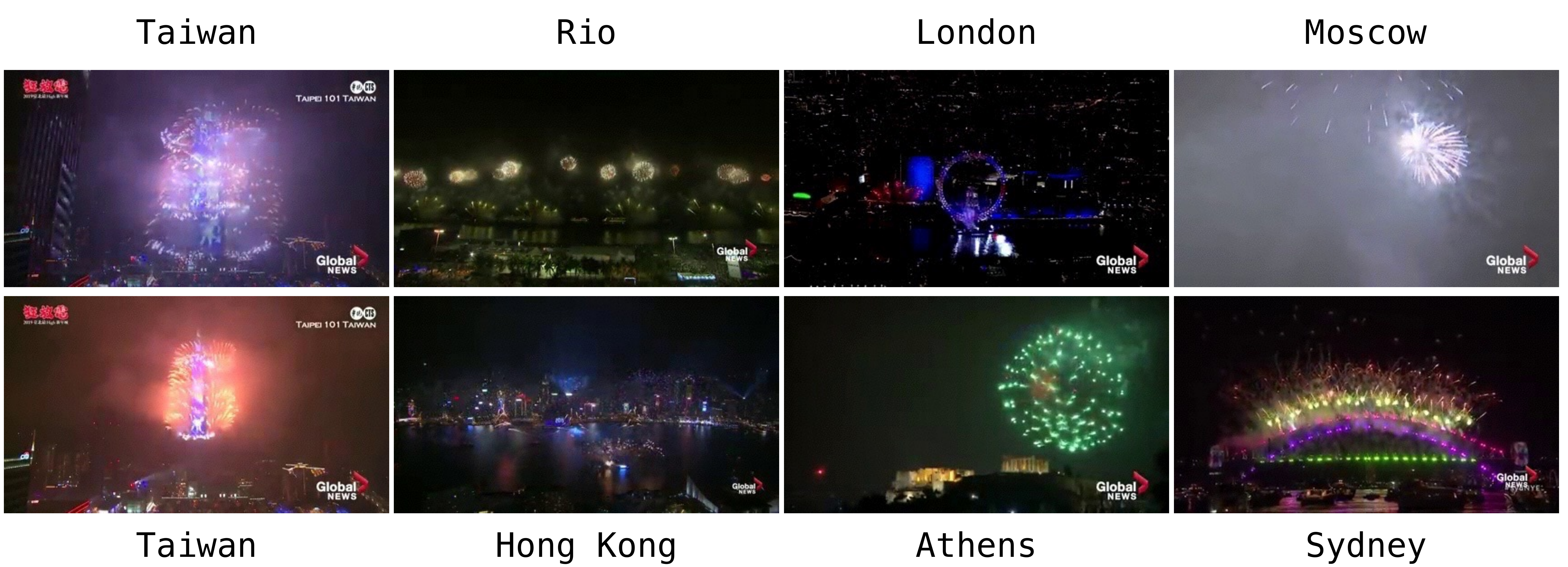}
	\vspace{-8pt}
\caption{Eight representative frames chosen by \AlgHybrid from ten different simultaneous feeds of New Year's Eve fireworks from around the world.}
\label{sampleFrames}
\end{figure}

\begin{figure}[tb!] 
	\centering     
\subfloat[]{\includegraphics[height=31mm]{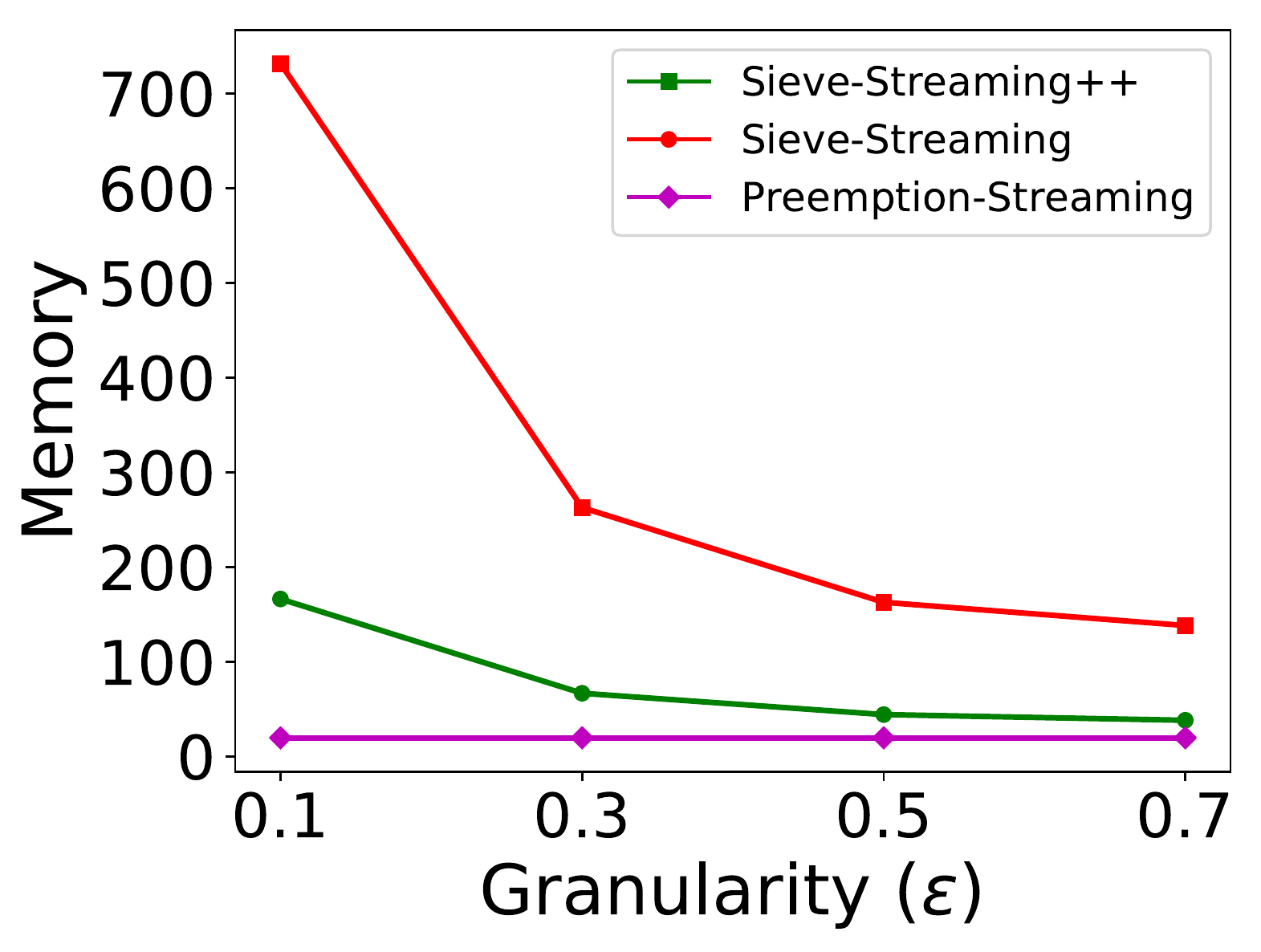}}
\subfloat[]{\includegraphics[height=31mm]{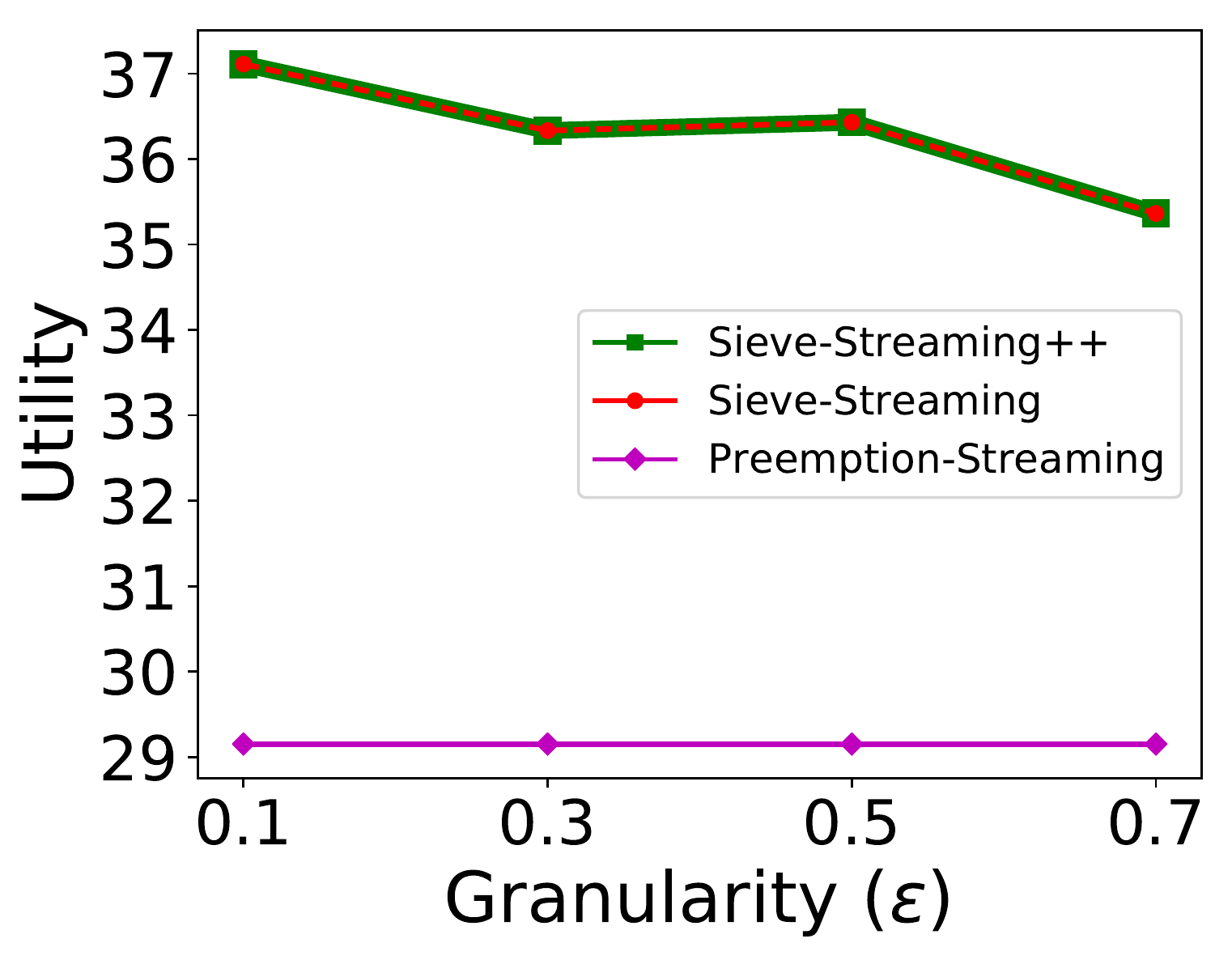}}
\subfloat[]{\includegraphics[height=31mm]{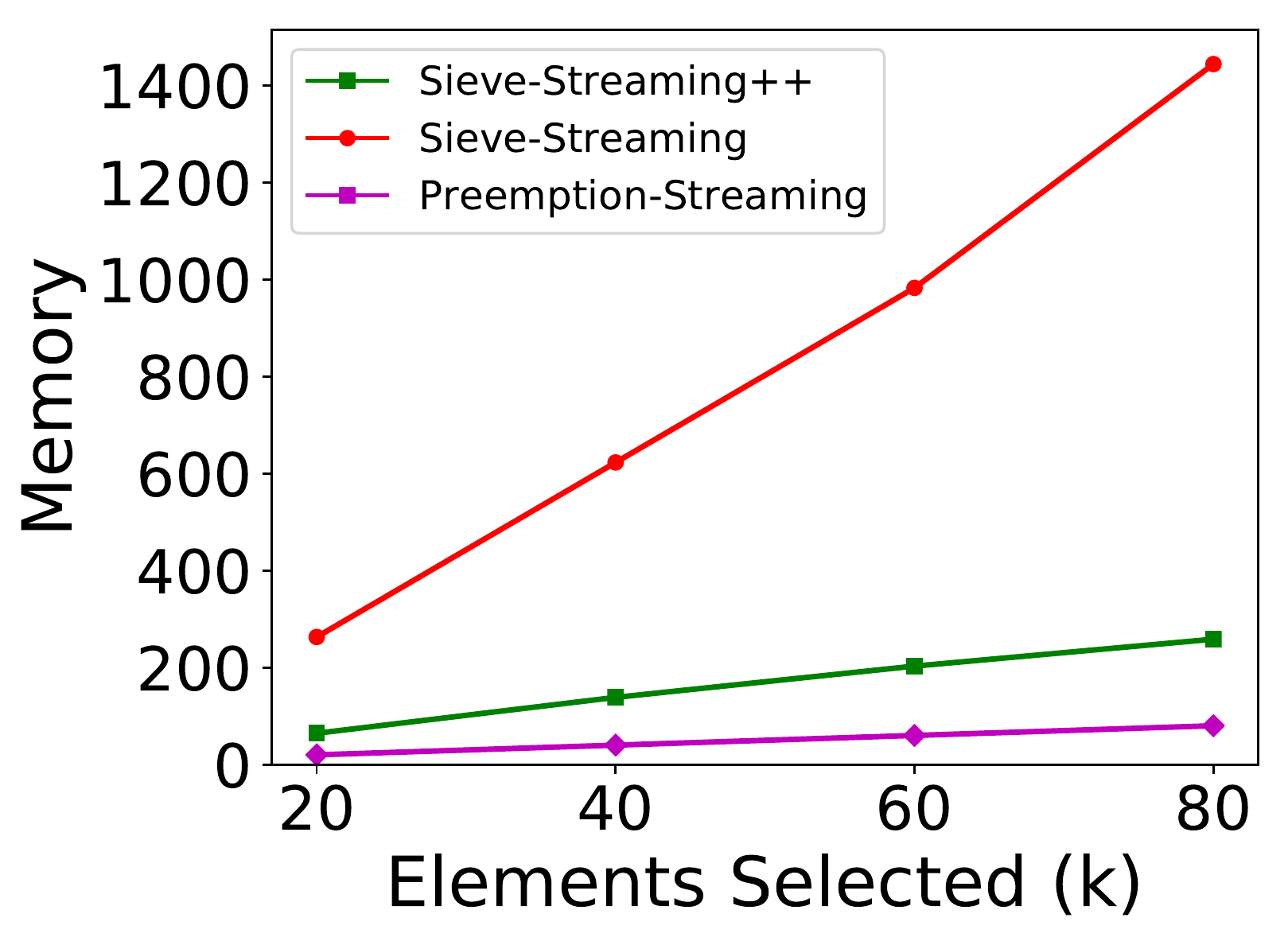}}
\subfloat[]{\includegraphics[height=31mm]{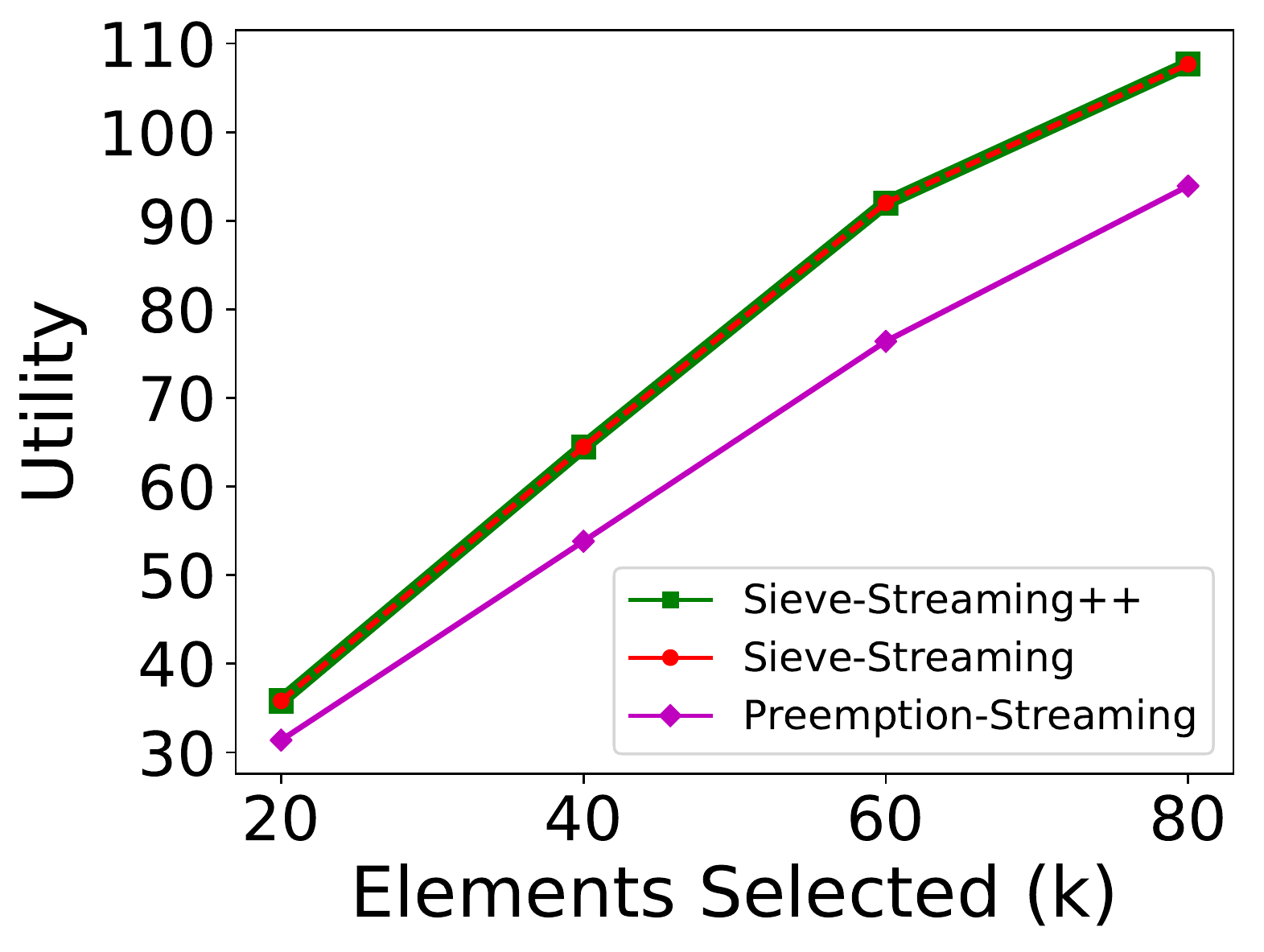}}
\caption{Graphs showing how the memory and utility of various single-source streaming algorithms vary with the cardinality $k$ and the granularity parameter $\epsilon$. Note that the utility of \AlgSieveStreamingPlus and \AlgSieveStreaming are exactly overlapping. In (a) and (b) we use k = 20, while in (c) and (d) we use $\epsilon = 0.3$.}
\label{sieveGraphs}
\end{figure}

\subsection{Single-Source Experiments} \label{singleSourceExperiments}

In this section, we want to emphasize the power of \AlgSieveStreamingPlus in the single-source streaming scenario. As discussed earlier, the two existing standard approaches for monotone $k$-cardinality submodular streaming are \AlgSieveStreaming and \AlgStreamingPreemption.

As mentioned in \cref{sec:streamplus}, \AlgSieveStreamingPlus theoretically has the best properties of both of these existing baselines, with optimal memory complexity and the optimal approximation guarantee. Figure \ref{sieveGraphs} shows the performance of these three algorithms on the YouTube video summarization task and confirms that this holds in practice as well.

For the purposes of this test, we simply combined the different video feeds into one single stream. We see that the memory required by \AlgSieveStreamingPlus is much smaller than the memory required by \AlgSieveStreaming, but it still achieves the exact same utility. Furthermore, the memory requirement of \AlgSieveStreamingPlus is within a constant factor of \textnormal{\textsc{Preemption-Streaming}}, while its utility is much better. 
The Twitter experiment gives similar results so those graphs are deferred to \cref{more_graphs}.

\subsection{Multi-Source Experiments} \label{multiSourceExperiments}

\begin{figure}[tb!] 
	\centering     
\subfloat[]{\includegraphics[height=30.5mm]{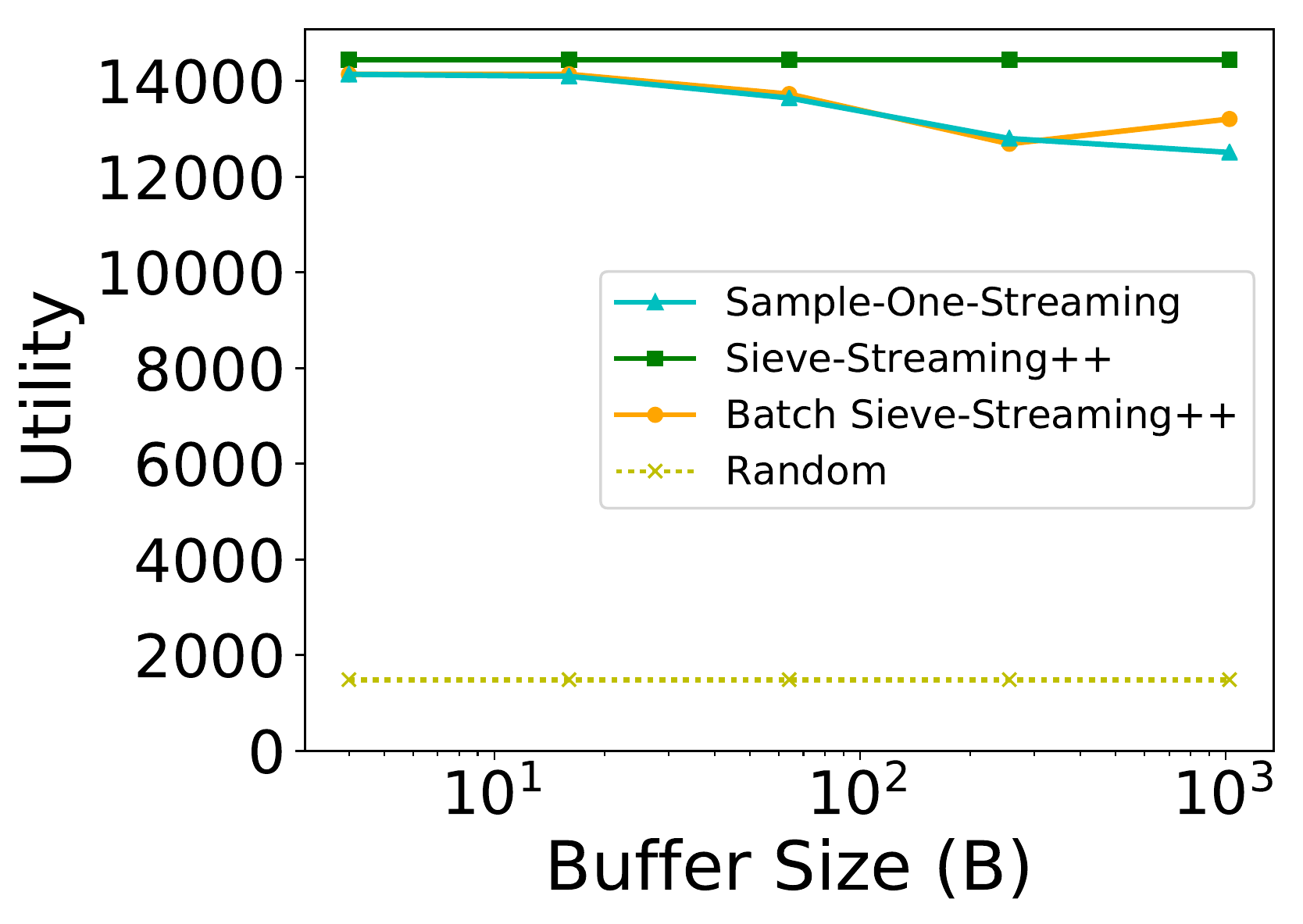}\label{twitterUtil}}
\subfloat[]{\includegraphics[height=30.5mm]{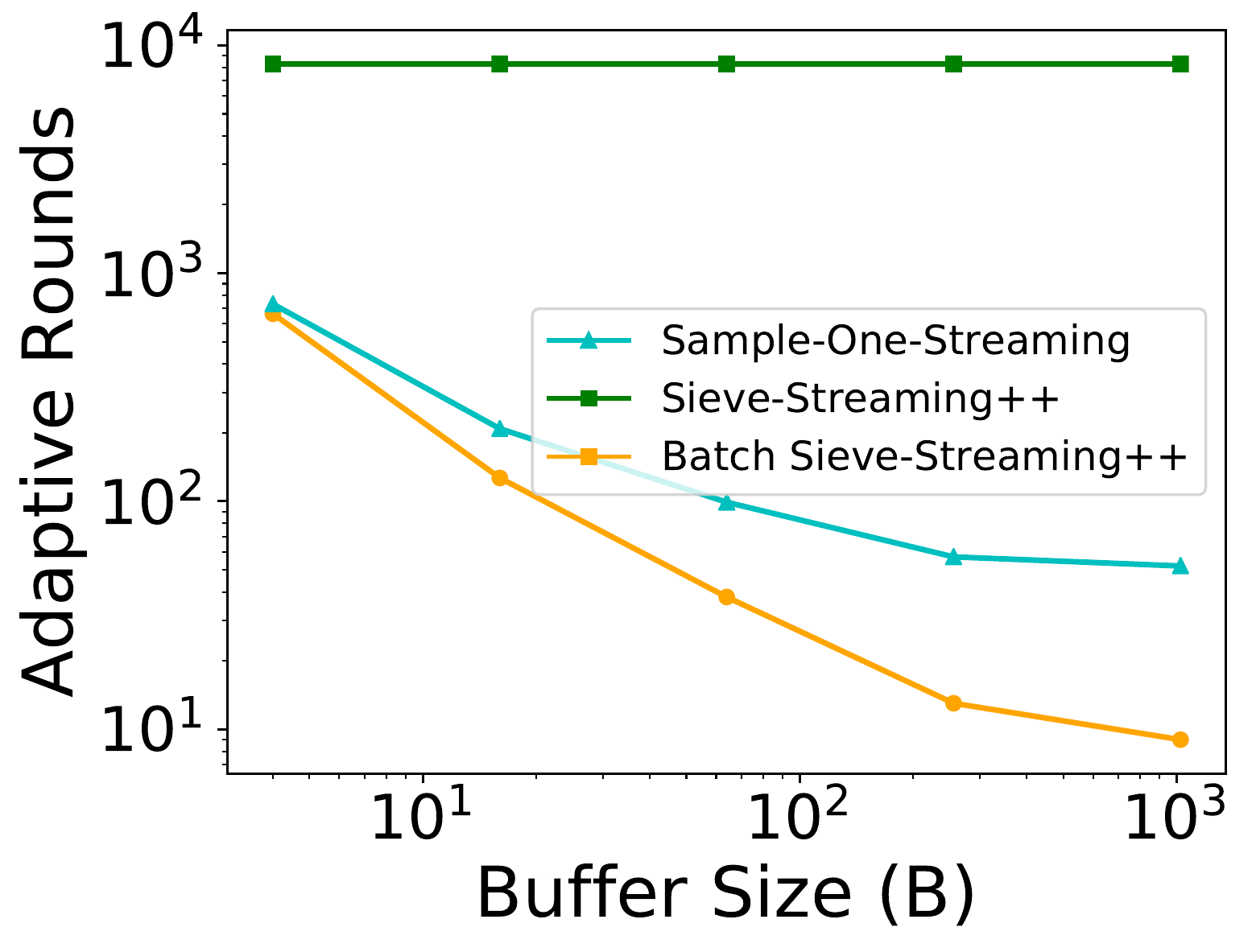}\label{twitterAdapt}}
\subfloat[]{\includegraphics[height=30.5mm]{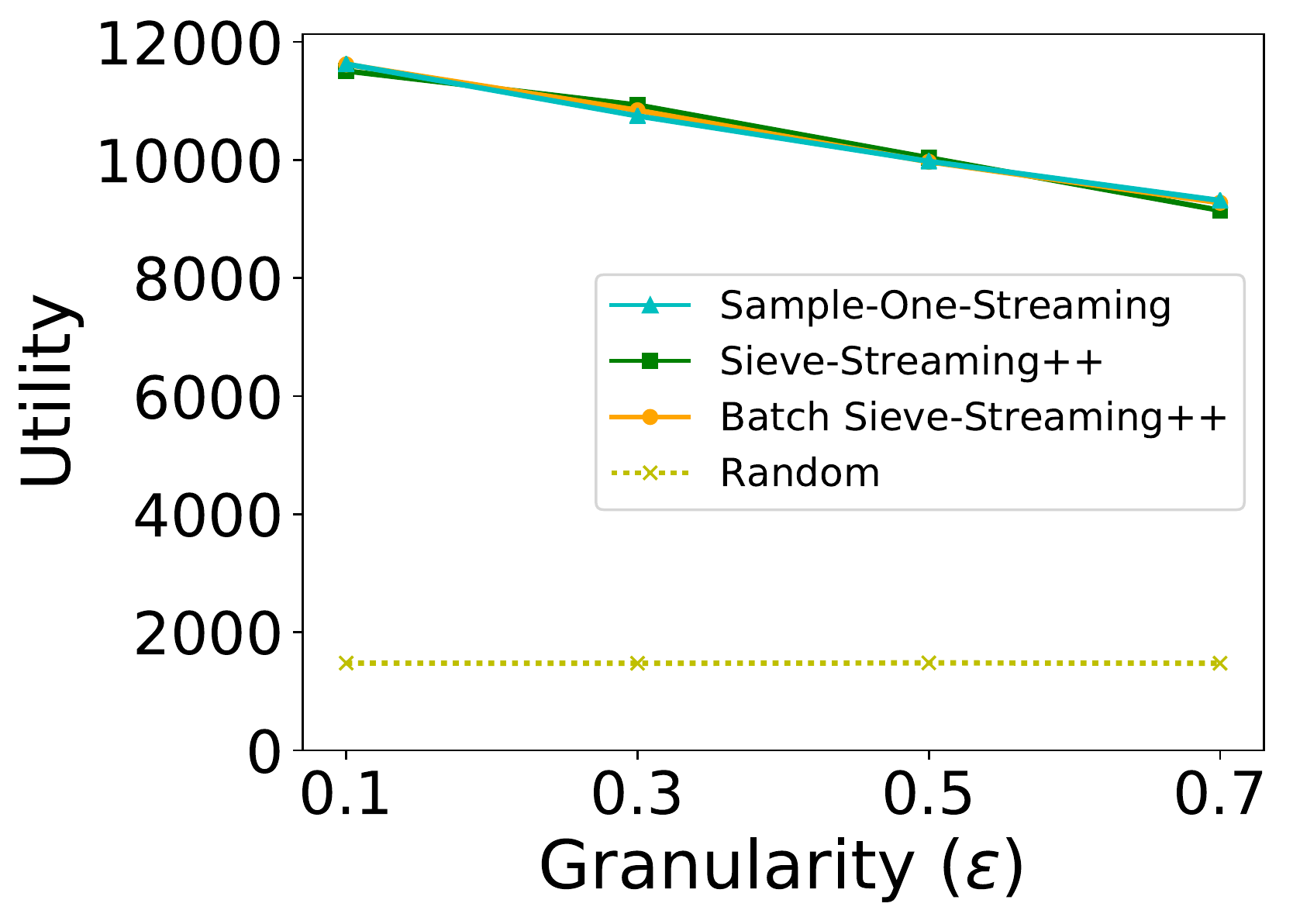}\label{varyEps1}}
\subfloat[]{\includegraphics[height=30.5mm]{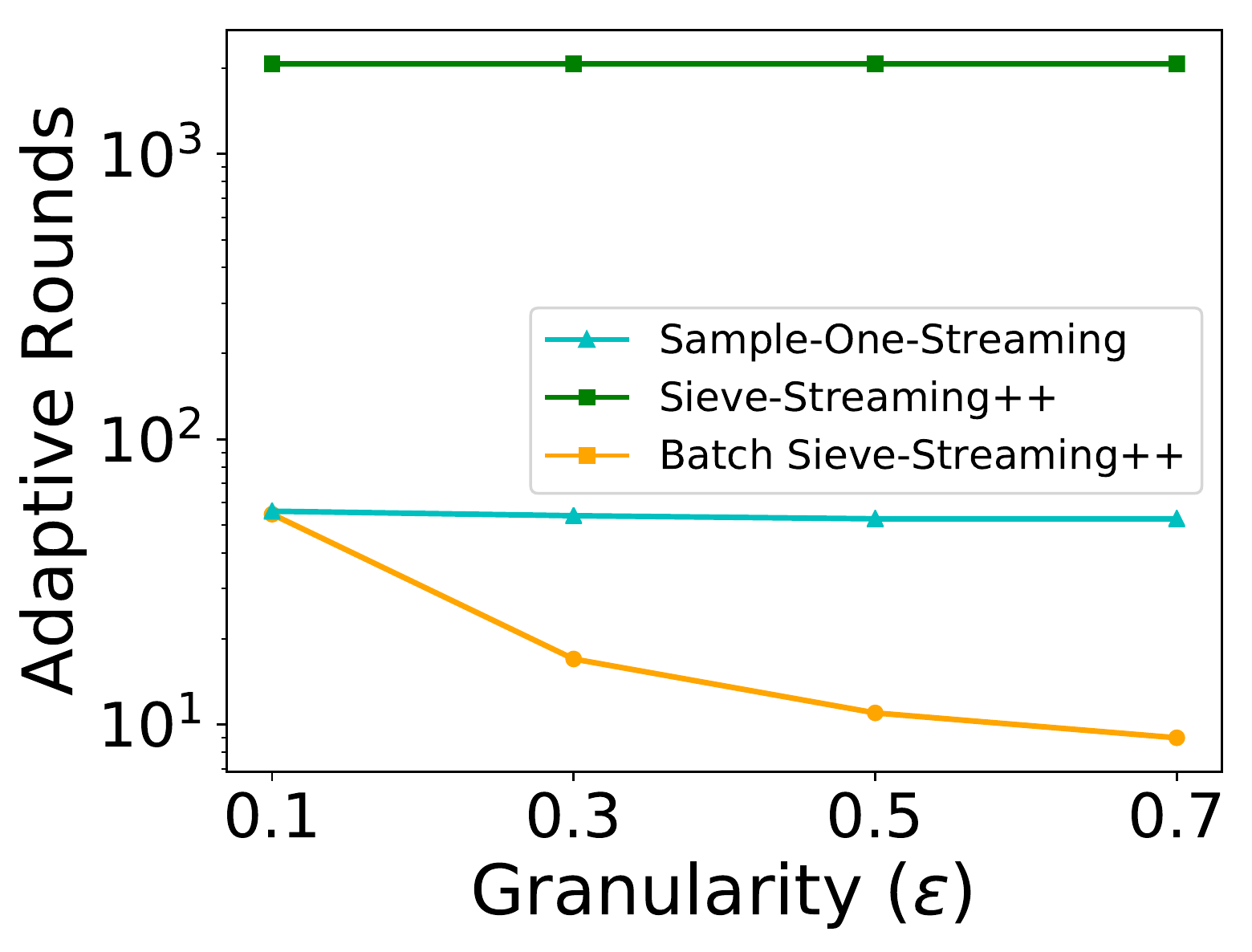}\label{varyEps2}}	
\caption{Graphs showing how the utility and adaptivity of various multi-source streaming algorithms vary with the buffer size $\memory$ and the granularity $\epsilon$. Unless they are being varied on the x-axis, we set $\epsilon = 0.7$, $\memory = 100$, and $k = 50$.}
\label{hybridGraphs}
\end{figure}

Once we move into the multi-source setting, the communication cost of algorithms becomes a key concern also. 
In this section, we compare the performance of algorithms in terms of utility and adaptivity where their communication cost is optimal.

Our first baseline is a trivial extension of \AlgSieveStreamingPlus. The multi-source extension for this algorithm essentially functions by locally computing the marginal gain of each incoming element, and only communicating it to the central machine if the marginal gain is above the desired threshold. However, as mentioned at the beginning \cref{batch}, this algorithm requires $\Omega(n)$ adaptive rounds. Our second baseline is \AlgOne, which was described in \cref{sec:single-alg}.

Figures \ref{twitterUtil} and \ref{twitterAdapt} show the effect of the buffer size $\memory$ on the performance of these algorithms for the Twitter task. The main observation is that \AlgHybrid can achieve roughly the same utility as the two baselines with many fewer adaptive rounds. Note that the number of adaptive rounds is shown in log scale.

Figures \ref{varyEps1} and \ref{varyEps2} show how these numbers vary with $\epsilon$. Again, the utilities of the three baselines are similar. We also see that increasing $\epsilon$ results in a large drop in the number of adaptive rounds for \AlgHybrid, but not for \AlgOne. \cref{more_graphs} gives some additional graphs, as well as the results for the YouTube dataset.

\subsection{Trade-off Between Communication and Adaptivity}

In the multi-source setting, there is a natural exchange between communication cost and adaptivity. 
Intuitively, the idea is that if we sample items more aggressively (which translates into higher communication cost), a set $S$ of $k$ items is generally picked faster, thus it reduces the adaptivity.
In the real world, the preference for one or the other can depend on a wide variety of factors ranging from resource constraints to the requirements of the particular problem.

In \AlgSampling,  we ensure the optimal communication performance
by sampling $t_i = \ceil{(1 + \epsilon)^{i+1} - (1+\epsilon)^{i}}$ items in each step of the for loop. 
Instead, to reduce the adaptivity by a factor of $\log(k)$, we could sample all the required $k$ items in a single step. 
Thus, in one adaptive round we mimic the two for loops of \AlgSampling. 
Doing this in each call to \cref{alg:Sampling} would reduce the expected adaptive complexity of \AlgSampling to the optimal $\log(\memory)$, but dramatically increase the communication cost to $O(k \log \memory)$.

In order to trade off between communication and adaptivity, we can instead sample $t^R_i = \ceil{(1 + \epsilon)^{i+R} - (1+\epsilon)^{i}}$ elements to perform $R$ consecutive adaptive rounds in only one round.
However, to maintain the same chance of a successful sampling, we still need to check the marginal gain. Finally, we pick a batch of the largest size $t^{j}_i$ such that the average marginal gain of the first $t^{j-1}_i$ items is above the desired threshold. Then we just add just this subset to $S_\tau$, meaning we have wasted $\ceil{(1 + \epsilon)^{i+R} - (1 + \epsilon)^{i+j}}$ communication.

Scatter plots of Figure~\ref{tradeoff} shows  how the number of adaptive rounds varies with the communication cost. 
Each individual dot represents a single run of the algorithm on a different subset of the data. The different colors cluster the dots into groups based on the value of $R$ that we used in that run.
Note that the parameter $R$ controls the communication cost.

The plot on the left comes from the Twitter experiment, while the plot on the right comes from the YouTube experiment. Although the shapes of the clusters are different in the two experiments,  we see that increasing $R$ increases the communication cost, but also decreases the number of adaptive rounds, as expected.

\begin{figure}[htb!] 
	\centering   
{\includegraphics[height=50mm]{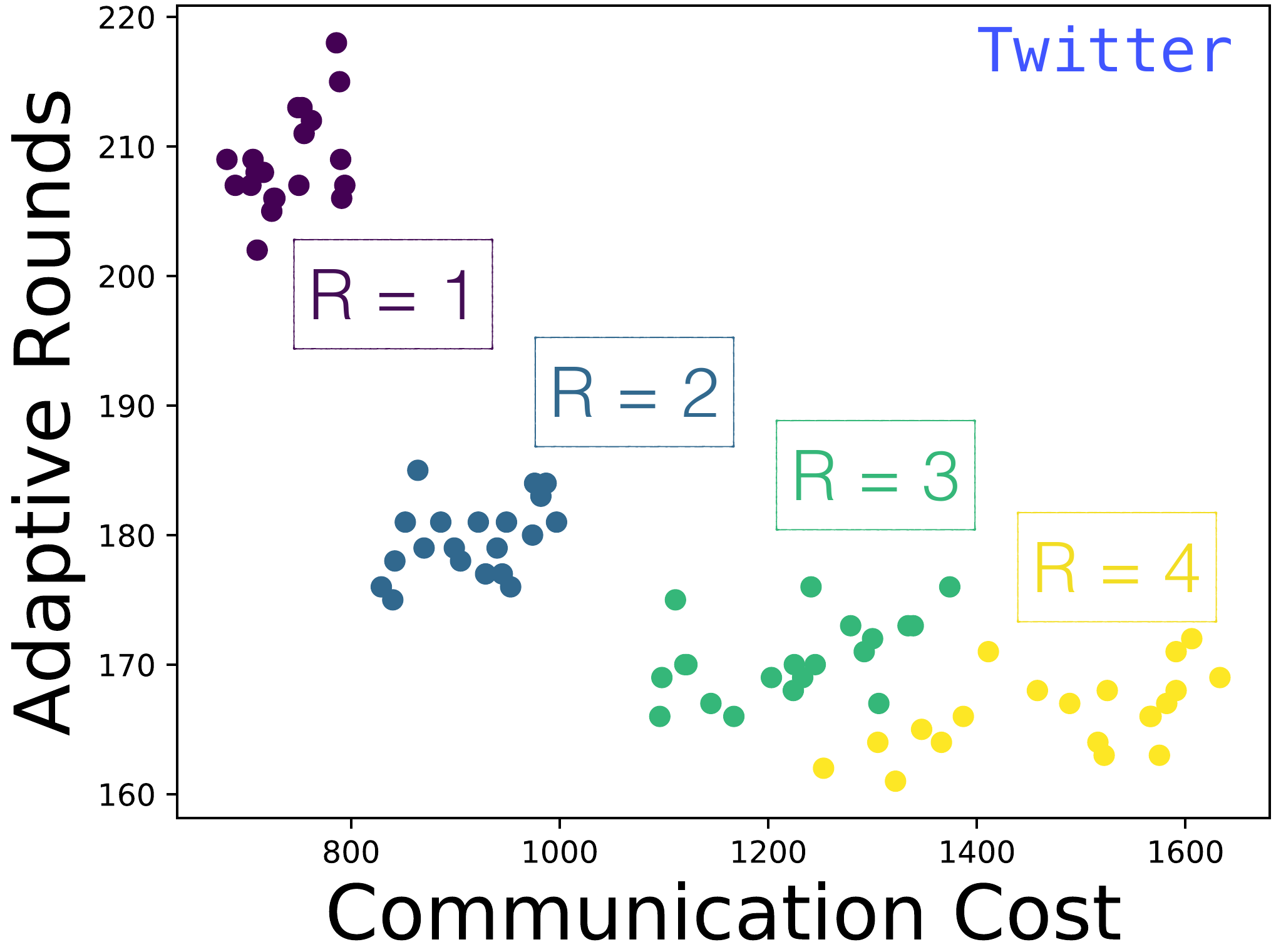}} 
\qquad
{\includegraphics[height=50mm]{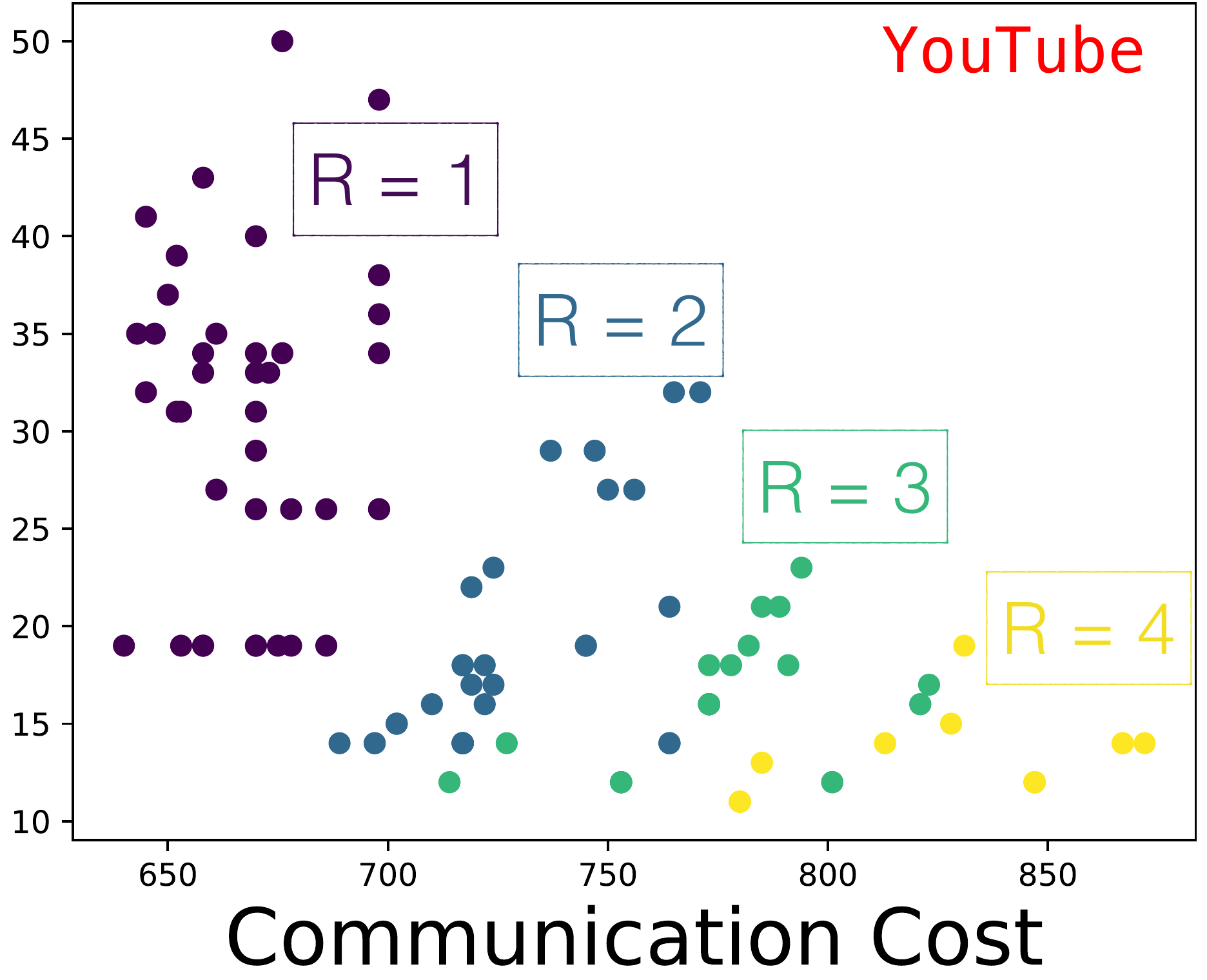}}
\caption{Scatter plots showing how we can lower the number of adaptive rounds by increasing communication. Each dot is the result of a single run of the algorithm and the colored clusters represent a particular setting for $R$. 
}
\label{tradeoff}
\end{figure}
\section{Implications of \AlgSieveStreamingPlus on Other Problems} \label{sec:discussion}

Recently, there has been several successful instances of using the sieving idea proposed by \citet{badanidiyuru2014streaming} for designing streaming algorithms for a wide range of submodular maximization problems.
In \cref{sec:streamplus}, we showed \AlgSieveStreamingPlus (see \cref{alg:Sieve-Stream++} and \cref{thm:sieve-stream++}) reduces the memory complexity of streaming submodular maximization to $O(k)$.
In this final section, we discuss how the approach of \AlgSieveStreamingPlus significantly improves the memory complexity for several previously studied important problems.

\paragraph{Random Order Streams}
\citet{norouzifard2018beyond} studied streaming submodular maximization under the assumption that elements of a stream arrive in random order.
They introduced a streaming algorithm called {{\textsc{SALSA}}\xspace} with an approximation guarantee better than $\nicefrac{1}{2}$.
This algorithm uses $O(k \log(k) )$ memory.
In a very straightforward way, similarly to the idea of \AlgSieveStreaming for lower bounding the optimum value,  we are able to improve the memory complexity of this algorithm to $O(k )$.
Furthermore,  \cite{norouzifard2018beyond} introduced a $p$-pass algorithm ($p\geq 2$) for submodular maximization subject to a cardinality constraint $k$.
We can also reduce the memory of this $p$-pass algorithm by a factor $\log (k)$.

\paragraph{Deletion-Robust}

\citet{mirzasoleiman2017deletion} have introduced a streaming algorithm for the deletion-robust submodular maximization. 
Their algorithm provides a summary $S$ of size $O(k d\log(k) / \epsilon)$ where it is robust to deletion of any set $D$ of at most $d$ items.
 \citet{kazemi2018scalable} were able to reduce the size of the deletion-robust summary to $O(k \log (k) / \epsilon + d \log^2(k) / \epsilon^3)$.
The idea of \AlgSieveStreamingPlus for estimating the value of \Opt reduces the memory complexities of these two algorithms to  $O(k d / \epsilon)$ and $O(k  / \epsilon + d \log^2(k) / \epsilon^3)$, respectively.
It is also possible to reduce the memory complexity of {{\textsc{STAR-T-GREEDY}}\xspace} \citep{mitrovic2017streaming} by at least a factor of $\log(k)$.

\paragraph{Two-Stage}
\citet{mitrovic2018data} introduced a streaming algorithm called {{\textsc{Replacement-Streaming}}\xspace}  for the two-stage submodular maximization problem which is originally proposed by \citep{balkanski2016tlearning, stan2017probabilistic}.
The memory complexity of {{\textsc{Replacement-Streaming}}\xspace} is $O(\frac{\ell \log \ell }{\epsilon})$, where $\ell$ is the size of the produced summary.
Again, by applying the idea of \AlgSieveStreamingPlus for guessing the value of \Opt and analysis similar to the proof of \cref{thm:sieve-stream++}, we can reduce the memory complexity of the streaming two-stage submodular maximization to $O(\frac{\ell}{\epsilon})$. 

\paragraph{Streaming Weak Submodularity}
Weak submodular functions generalize the diminishing returns property.
\begin{definition}[Weakly Submodular \citep{das2011submodular}]
	A monotone and non-negative set function $f: 2^\ground \rightarrow \bR_{\geq 0}$ is  $\gamma$--weakly submodular if for each sets $A,B \subset \ground$, we have
	\begin{align*}
		\gamma \leq  \frac{\sum_{e \in A \setminus B} f(\{e\} \mid B)}{f(A \mid B)},
	\end{align*}
	where the ratio is considered to be equal to $1$ when its numerator and denominator are $0$.
\end{definition}
It is straightforward to show that $f$ is submodular if and only if $\gamma = 1$.
In the streaming context subject to a cardinality constraint $k$, \citet{elenberg2017streaming} designed an algorithm with a constant factor approximation for $\gamma$--weakly submodular functions.
The memory complexity of their algorithm is $O(\frac{k \log k}{\epsilon})$.
By adopting the idea of \AlgSieveStreamingPlus, we could reduce the memory complexity of their algorithm to $O(\frac{k}{\epsilon})$.

\cref{tbl:other-summary} provides a summary of algorithms that we could significantly improve their memory complexity, while their approximation factors are maintained.

\begin{table}
	\caption{Improved results for several other recently studied submodular maximization problems. The main idea of \AlgSieveStreamingPlus enables us to significantly reduce the memory complexity of these problems.} \label{tbl:other-summary}
	{\scriptsize
	\begin{center}
		\begin{tabular}{llccl}
			\toprule
		\textbf{Problem} &	\textbf{Algorithm} & \textbf{Old Memory} & \textbf{New Memory} & \textbf{Ref.} \\
			\midrule
			Streaming & {\textsc{Salsa}} & $O(k \log(k))$ & $O(k)$  & \citet{norouzifard2018beyond} \\
			Streaming & {\textsc{P-PASS}} & $O(\nicefrac{k \log (k)}{\epsilon})$ & $O(\nicefrac{k}{\epsilon})$  & \citet{norouzifard2018beyond} \\
				\midrule
				Weak submodular & {\textsc{Streak}}  & $O(\nicefrac{k \log (k)}{\epsilon})$ & $O(\nicefrac{k}{\epsilon})$ & \citet{elenberg2017streaming}\\
				\midrule 
				Deletion-Robust & {\textsc{ROBUST}} &  $O(\nicefrac{k d \log(k)} {\epsilon})$ & $O(\nicefrac{k d} {\epsilon})$ & \citet{mirzasoleiman2017deletion} \\
				Deletion-Robust & {\textsc{ROBUST-STREAM}} & $O(\nicefrac{k \log (k) }{\epsilon} + \nicefrac{d \log^2(k)}{\epsilon^3)}$ &  $O(\nicefrac{k }{\epsilon} + \nicefrac{d \log^2(k)}{\epsilon^3)}$ & \citet{kazemi2018scalable} \\
				\midrule
				Two-Stage& {\textsc{REPLACE-STREAM}} &  $O(\nicefrac{\ell \log (\ell) }{\epsilon})$ &  $O(\nicefrac{\ell}{\epsilon})$&  \citet{mitrovic2018data}\\
			\bottomrule
		\end{tabular}
	\end{center}
}
\end{table}

\section{Conclusion}
In this paper, we studied the problem of maximizing a non-negative submodular function over a multi-source stream of data subject to a cardinality constraint $k$.
We first proposed \AlgSieveStreamingPlus with the optimum approximation factor and memory complexity for a single stream of data.
Build upon this idea, we designed an algorithm for multi-source streaming setting with a $\nicefrac{1}{2}$ approximation factor, $O(k)$ memory complexity, a very low communication cost, and near-optimal adaptivity.
We evaluated the performance of our algorithms on two real-world data sets of multi-source tweet streams and video streams.
Furthermore, by using the main idea of \AlgSieveStreamingPlus, we  significantly improved the memory complexity of several important submodular maximization problems.

\paragraph{Acknowledgements.} The work of Amin Karbasi is supported by AFOSR Young Investigator Award (FA9550-18-1-0160) and  Grainger Award (PO 2000008083 2016012).
We would like to thank Ola Svensson for his comment on the first version of this manuscript.
\bibliographystyle{plainnat}
\bibliography{hybridStreaming}

\begin{thebibliography}{43}
\providecommand{\natexlab}[1]{#1}
\providecommand{\url}[1]{\texttt{#1}}
\expandafter\ifx\csname urlstyle\endcsname\relax
  \providecommand{\doi}[1]{doi: #1}\else
  \providecommand{\doi}{doi: \begingroup \urlstyle{rm}\Url}\fi

\bibitem[Badanidiyuru and Vondr{\'{a}}k(2014)]{badanidiyuru2014fast}
Ashwinkumar Badanidiyuru and Jan Vondr{\'{a}}k.
\newblock {Fast algorithms for maximizing submodular functions}.
\newblock In \emph{Symposium on Discrete Algorithms, {SODA}}, pages 1497--1514,
  2014.

\bibitem[Badanidiyuru et~al.(2014)Badanidiyuru, Mirzasoleiman, Karbasi, and
  Krause]{badanidiyuru2014streaming}
Ashwinkumar Badanidiyuru, Baharan Mirzasoleiman, Amin Karbasi, and Andreas
  Krause.
\newblock {Streaming Submodular Maximization:Massive Data Summarization on the
  Fly}.
\newblock In \emph{International Conference on Knowledge Discovery and Data
  Mining, {KDD}}, pages 671--680, 2014.

\bibitem[Balkanski and Singer(2018)]{balkanski2018adaptive}
Eric Balkanski and Yaron Singer.
\newblock {The adaptive complexity of maximizing a submodular function}.
\newblock In \emph{{ACM} {SIGACT} Symposium on Theory of Computing, {STOC}},
  pages 1138--1151, 2018.

\bibitem[Balkanski et~al.(2016)Balkanski, Mirzasoleiman, Krause, and
  Singer]{balkanski2016tlearning}
Eric Balkanski, Baharan Mirzasoleiman, Andreas Krause, and Yaron Singer.
\newblock Learning sparse combinatorial representations via two-stage
  submodular maximization.
\newblock In \emph{International Conference on Machine Learning (ICML)}, 2016.

\bibitem[Balkanski et~al.(2018)Balkanski, Breuer, and
  Singer]{balkanski2018nonmonotone}
Eric Balkanski, Adam Breuer, and Yaron Singer.
\newblock {Non-monotone Submodular Maximization in Exponentially Fewer
  Iterations}.
\newblock In \emph{Advances in Neural Information Processing Systems}, pages
  2359--2370, 2018.

\bibitem[Balkanski et~al.(2019)Balkanski, Rubinstein, and
  Singer]{balkanski2018exponential}
Eric Balkanski, Aviad Rubinstein, and Yaron Singer.
\newblock {An Exponential Speedup in Parallel Running Time for Submodular
  Maximization without Loss in Approximation}.
\newblock In \emph{Symposium on Discrete Algorithms (SODA)}, pages 283--302,
  2019.

\bibitem[Bansal and Sviridenko(2006)]{bansal2006santa}
Nikhil Bansal and Maxim Sviridenko.
\newblock The santa claus problem.
\newblock In \emph{Proceedings of the Thirty-Eighth Annual ACM Symposium on
  Theory of Computing}, pages 31--40. ACM, 2006.

\bibitem[Barbosa et~al.(2015)Barbosa, Ene, Nguyen, and Ward]{barbosa2015power}
Rafael Barbosa, Alina Ene, Huy Nguyen, and Justin Ward.
\newblock The power of randomization: Distributed submodular maximization on
  massive datasets.
\newblock In \emph{International Conference on Machine Learning (ICML)}, pages
  1236--1244, 2015.

\bibitem[Barbosa et~al.(2016)Barbosa, Ene, Nguyen, and Ward]{barbosa2016new}
Rafael Barbosa, Alina Ene, Huy~L. Nguyen, and Justin Ward.
\newblock {A New Framework for Distributed Submodular Maximization}.
\newblock In \emph{Annual Symposium on Foundations of Computer Science,
  {FOCS}}, pages 645--654, 2016.

\bibitem[Buchbinder et~al.(2015)Buchbinder, Feldman, and
  Schwartz]{buchbinder2015online}
Niv Buchbinder, Moran Feldman, and Roy Schwartz.
\newblock Online submodular maximization with preemption.
\newblock In \emph{{ACM-SIAM} Symposium on Discrete Algorithms, {SODA}}, pages
  1202--1216, 2015.

\bibitem[Buchbinder et~al.(2017)Buchbinder, Feldman, and
  Schwartz]{buchbinder2017comparing}
Niv Buchbinder, Moran Feldman, and Roy Schwartz.
\newblock {Comparing Apples and Oranges: Query Trade-off in Submodular
  Maximization}.
\newblock \emph{Math. Oper. Res.}, 42\penalty0 (2):\penalty0 308--329, 2017.

\bibitem[Chakrabarti and Kale(2015)]{chakrabarti2015submodular}
Amit Chakrabarti and Sagar Kale.
\newblock {Submodular maximization meets streaming: matchings, matroids, and
  more}.
\newblock \emph{Math. Program.}, 154\penalty0 (1-2):\penalty0 225--247, 2015.

\bibitem[Chan et~al.(2016)Chan, Huang, Jiang, Kang, and Tang]{chan2017online}
TH~Chan, Zhiyi Huang, Shaofeng H-C Jiang, Ning Kang, and Zhihao~Gavin Tang.
\newblock Online submodular maximization with free disposal: Randomization
  beats 0.25 for partition matroids.
\newblock 2016.

\bibitem[Chekuri and Quanrud(2018)]{chekuri2018parallelizing}
Chandra Chekuri and Kent Quanrud.
\newblock {Parallelizing greedy for submodular set function maximization in
  matroids and beyond}.
\newblock \emph{CoRR}, abs/1811.12568, 2018.

\bibitem[Chekuri et~al.(2015)Chekuri, Gupta, and Quanrud]{chekuri2015streaming}
Chandra Chekuri, Shalmoli Gupta, and Kent Quanrud.
\newblock Streaming algorithms for submodular function maximization.
\newblock In \emph{International Colloquium on Automata, Languages, and
  Programming}, pages 318--330. Springer, 2015.

\bibitem[Chen et~al.(2016)Chen, Nguyen, and Zhang]{chen2016submodular}
Jiecao Chen, Huy~L. Nguyen, and Qin Zhang.
\newblock {Submodular Maximization over Sliding Windows}.
\newblock \emph{CoRR}, abs/1611.00129, 2016.

\bibitem[Chen et~al.(2018)Chen, Feldman, and Karbasi]{chen2018unconstrained}
Lin Chen, Moran Feldman, and Amin Karbasi.
\newblock Unconstrained submodular maximization with constant adaptive
  complexity.
\newblock \emph{CoRR}, abs/1811.06603, 2018.

\bibitem[Das and Kempe(2011)]{das2011submodular}
Abhimanyu Das and David Kempe.
\newblock {Submodular meets Spectral: Greedy Algorithms for Subset Selection,
  Sparse Approximation and Dictionary Selection}.
\newblock In \emph{International Conference on Machine Learning (ICML)}, pages
  1057--1064, 2011.

\bibitem[Elenberg et~al.(2017)Elenberg, Dimakis, Feldman, and
  Karbasi]{elenberg2017streaming}
Ethan~R. Elenberg, Alexandros~G. Dimakis, Moran Feldman, and Amin Karbasi.
\newblock {Streaming Weak Submodularity: Interpreting Neural Networks on the
  Fly}.
\newblock In \emph{Advances in Neural Information Processing Systems}, pages
  4047--4057, 2017.

\bibitem[Ene and Nguyen(2019)]{ene2018submodular}
Alina Ene and Huy~L. Nguyen.
\newblock {Submodular Maximization with Nearly-optimal Approximation and
  Adaptivity in Nearly-linear Time}.
\newblock In \emph{Symposium on Discrete Algorithms (SODA)}, pages 274--282,
  2019.

\bibitem[Epasto et~al.(2017)Epasto, Lattanzi, Vassilvitskii, and
  Zadimoghaddam]{epasto2017submodular}
Alessandro Epasto, Silvio Lattanzi, Sergei Vassilvitskii, and Morteza
  Zadimoghaddam.
\newblock {Submodular Optimization Over Sliding Windows}.
\newblock In \emph{WWW}, pages 421--430, 2017.

\bibitem[Fahrbach et~al.(2018)Fahrbach, Mirrokni, and
  Zadimoghaddam]{fahrbach2018nonmonotone}
Matthew Fahrbach, Vahab~S. Mirrokni, and Morteza Zadimoghaddam.
\newblock {Non-monotone Submodular Maximization with Nearly Optimal Adaptivity
  Complexity}.
\newblock \emph{CoRR}, abs/1808.06932, 2018.

\bibitem[Fahrbach et~al.(2019)Fahrbach, Mirrokni, and
  Zadimoghaddam]{fahrbach2018submodular}
Matthew Fahrbach, Vahab~S. Mirrokni, and Morteza Zadimoghaddam.
\newblock {Submodular Maximization with Nearly Optimal Approximation,
  Adaptivity and Query Complexity}.
\newblock In \emph{Symposium on Discrete Algorithms (SODA)}, pages 255--273,
  2019.

\bibitem[Feldman et~al.(2017)Feldman, Harshaw, and Karbasi]{feldman2017greed}
Moran Feldman, Christopher Harshaw, and Amin Karbasi.
\newblock {Greed Is Good: Near-Optimal Submodular Maximization via Greedy
  Optimization}.
\newblock In \emph{{Conference on Learning Theory}}, 2017.

\bibitem[Feldman et~al.(2018)Feldman, Karbasi, and Kazemi]{feldman2018do}
Moran Feldman, Amin Karbasi, and Ehsan Kazemi.
\newblock {Do Less, Get More: Streaming Submodular Maximization with
  Subsampling}.
\newblock In \emph{Advances in Neural Information Processing Systems}, pages
  730--740, 2018.

\bibitem[Herbrich et~al.(2003)Herbrich, Lawrence, and Seeger]{herbrich2003fast}
Ralf Herbrich, Neil~D Lawrence, and Matthias Seeger.
\newblock Fast sparse gaussian process methods: The informative vector machine.
\newblock In \emph{Advances in Neural Information Processing Systems}, pages
  625--632, 2003.

\bibitem[Kazemi et~al.(2018)Kazemi, Zadimoghaddam, and
  Karbasi]{kazemi2018scalable}
Ehsan Kazemi, Morteza Zadimoghaddam, and Amin Karbasi.
\newblock {Scalable Deletion-Robust Submodular Maximization: Data Summarization
  with Privacy and Fairness Constraints}.
\newblock In \emph{International Conference on Machine Learning (ICML)}, pages
  2549--2558, 2018.

\bibitem[Krause and Golovin(2012)]{krause12survey}
Andreas Krause and Daniel Golovin.
\newblock {Submodular Function Maximization}.
\newblock In \emph{Tractability: Practical Approaches to Hard Problems}.
  Cambridge University Press, 2012.

\bibitem[Kumar et~al.(2015)Kumar, Moseley, Vassilvitskii, and
  Vattani]{kumar2015fast}
Ravi Kumar, Benjamin Moseley, Sergei Vassilvitskii, and Andrea Vattani.
\newblock {Fast Greedy Algorithms in MapReduce and Streaming}.
\newblock \emph{{TOPC}}, 2\penalty0 (3):\penalty0 14:1--14:22, 2015.

\bibitem[Liu and Vondr{\'{a}}k(2018)]{liu2018submodular}
Paul Liu and Jan Vondr{\'{a}}k.
\newblock {Submodular Optimization in the MapReduce Model}.
\newblock \emph{CoRR}, abs/1810.01489, 2018.

\bibitem[Mirrokni and Zadimoghaddam(2015)]{mirrokni2015randomized}
Vahab Mirrokni and Morteza Zadimoghaddam.
\newblock {Randomized composable core-sets for distributed submodular
  maximization}.
\newblock In \emph{ACM on Symposium on Theory of Computing, , {STOC}}, pages
  153--162. ACM, 2015.

\bibitem[Mirzasoleiman et~al.(2015)Mirzasoleiman, Badanidiyuru, Karbasi,
  Vondrak, and Krause]{mirzasoleiman2015lazier}
Baharan Mirzasoleiman, Ashwinkumar Badanidiyuru, Amin Karbasi, Jan Vondrak, and
  Andreas Krause.
\newblock {Lazier than Lazy Greedy}.
\newblock In \emph{AAAI Conference on Artificial Intelligence}, pages
  1812--1818, 2015.

\bibitem[Mirzasoleiman et~al.(2016{\natexlab{a}})Mirzasoleiman, Badanidiyuru,
  and Karbasi]{mirzasoleiman2016fast}
Baharan Mirzasoleiman, Ashwinkumar Badanidiyuru, and Amin Karbasi.
\newblock Fast constrained submodular maximization: Personalized data
  summarization.
\newblock In \emph{International Conference on Machine Learning (ICML)}, pages
  1358--1367, 2016{\natexlab{a}}.

\bibitem[Mirzasoleiman et~al.(2016{\natexlab{b}})Mirzasoleiman, Karbasi,
  Sarkar, and Krause]{mirzasoleiman16distributed}
Baharan Mirzasoleiman, Amin Karbasi, Rik Sarkar, and Andreas Krause.
\newblock {Distributed Submodular Maximization}.
\newblock \emph{Journal of Machine Learning Research (JMLR)}, 17:\penalty0
  1--44, 2016{\natexlab{b}}.

\bibitem[Mirzasoleiman et~al.(2016{\natexlab{c}})Mirzasoleiman, Zadimoghaddam,
  and Karbasi]{mirzasoleiman2016cover}
Baharan Mirzasoleiman, Morteza Zadimoghaddam, and Amin Karbasi.
\newblock {Fast Distributed Submodular Cover: Public-Private Data
  Summarization}.
\newblock In \emph{Advances in Neural Information Processing Systems},
  2016{\natexlab{c}}.

\bibitem[Mirzasoleiman et~al.(2017)Mirzasoleiman, Karbasi, and
  Krause]{mirzasoleiman2017deletion}
Baharan Mirzasoleiman, Amin Karbasi, and Andreas Krause.
\newblock {Deletion-Robust Submodular Maximization: Data Summarization with
  ``the Right to be Forgotten''}.
\newblock In \emph{International Conference on Machine Learning (ICML)}, pages
  2449--2458, 2017.

\bibitem[Mirzasoleiman et~al.(2018)Mirzasoleiman, Jegelka, and
  Krause]{mirzasoleiman2018streaming}
Baharan Mirzasoleiman, Stefanie Jegelka, and Andreas Krause.
\newblock {Streaming Non-Monotone Submodular Maximization: Personalized Video
  Summarization on the Fly}.
\newblock In \emph{{AAAI} Conference on Artificial Intelligence}, 2018.

\bibitem[Mitrovic et~al.(2017{\natexlab{a}})Mitrovic, Bun, Krause, and
  Karbasi]{mitrovic2017differentially}
Marko Mitrovic, Mark Bun, Andreas Krause, and Amin Karbasi.
\newblock {Differentially Private Submodular Maximization: Data Summarization
  in Disguise}.
\newblock In \emph{International Conference on Machine Learning (ICML)}, pages
  2478--2487, 2017{\natexlab{a}}.

\bibitem[Mitrovic et~al.(2018)Mitrovic, Kazemi, Zadimoghaddam, and
  Karbasi]{mitrovic2018data}
Marko Mitrovic, Ehsan Kazemi, Morteza Zadimoghaddam, and Amin Karbasi.
\newblock {Data Summarization at Scale: {A} Two-Stage Submodular Approach}.
\newblock In \emph{International Conference on Machine Learning (ICML)}, pages
  3593--3602, 2018.

\bibitem[Mitrovic et~al.(2017{\natexlab{b}})Mitrovic, Bogunovic, Norouzi-Fard,
  Tarnawski, and Cevher]{mitrovic2017streaming}
Slobodan Mitrovic, Ilija Bogunovic, Ashkan Norouzi-Fard, Jakub~M Tarnawski, and
  Volkan Cevher.
\newblock {Streaming Robust Submodular Maximization: A Partitioned Thresholding
  Approach}.
\newblock In \emph{Advances in Neural Information Processing Systems}, pages
  4560--4569, 2017{\natexlab{b}}.

\bibitem[Nemhauser et~al.(1978)Nemhauser, Wolsey, and
  Fisher]{nemhauser1978analysis}
George~L Nemhauser, Laurence~A Wolsey, and Marshall~L Fisher.
\newblock {An analysis of approximations for maximizing submodular set
  functions-I}.
\newblock \emph{Mathematical programming}, 14\penalty0 (1):\penalty0 265--294,
  1978.

\bibitem[Norouzi{-}Fard et~al.(2018)Norouzi{-}Fard, Tarnawski, Mitrovic,
  Zandieh, Mousavifar, and Svensson]{norouzifard2018beyond}
Ashkan Norouzi{-}Fard, Jakub Tarnawski, Slobodan Mitrovic, Amir Zandieh,
  Aidasadat Mousavifar, and Ola Svensson.
\newblock {Beyond 1/2-Approximation for Submodular Maximization on Massive Data
  Streams}.
\newblock In \emph{International Conference on Machine Learning (ICML)}, pages
  3826--3835, 2018.

\bibitem[Stan et~al.(2017)Stan, Zadimoghaddam, Krause, and
  Karbasi]{stan2017probabilistic}
Serban Stan, Morteza Zadimoghaddam, Andreas Krause, and Amin Karbasi.
\newblock {Probabilistic Submodular Maximization in Sub-Linear Time}.
\newblock In \emph{International Conference on Machine Learning (ICML)}, 2017.

\end{thebibliography}

\newpage
\appendix
\section{Proof of \cref{lemma:filtering-rounds}} \label{sec:lemma-proof}
\begin{proof}
Since we are only adding elements to $S$, using submodularity the marginal value of any element $x$ to $S$, i.e. $f(x \mid S)$,  is decreasing. 
Therefore, once an element is removed from the buffer $\buffer$, it never comes back. As a result, the set $\buffer$ is shrinking over time. When $\buffer$ becomes empty, the algorithm terminates. Therefore it suffices to show that in  every iteration of the while loop, a constant fraction of elements will be removed from $\buffer$ in expectation. The rest of the analysis follows by analyzing the expected size of $\buffer$ over time and applying Markov's inequality. 

We note that to avoid confusion, we call one round of the while loop in Lines~\ref{line:filter}--\ref{line:end-loop} of \AlgSampling an iteration. 
There are two other internal for loops at Lines~\ref{line:loop1:begin}--\ref{line:loop1:end} and Lines~\ref{line:loop2:begin}--\ref{line:loop2:end}. Later in the proof, we call each run of these for loops a step. There are $\ceil{\frac{1}{\epsilon}}$ steps in the first for loop and $O(\log(k))$ steps in the second.

If an iteration ends with growing $S$ into a size $k$ set, that is going to be the final iteration as the algorithm  \AlgSampling because the algorithm terminates once $k$ elements are selected. So we focus on the other case. An iteration breaks (finishes) either in the first for the loop at Lines~\ref{line:loop1:begin}--\ref{line:loop1:end} or in the second for loop of  Lines~\ref{line:loop2:begin}--\ref{line:loop2:end}.
We say an iteration fails if after termination less than $\epsilon/2$ fraction of elements in $\buffer$ is removed. 
For iteration ${\ell}$, let $\buffer^{\ell}$ be the set $\buffer$ at Line~\ref{line:filter} at the beginning of this iteration. So the first set $\buffer^1$ consists of all the input elements passed to \AlgSampling. 
So we can say that an iteration ${\ell}$ fails if $|\buffer^{{\ell}+1}|$ is greater than $(1-\epsilon/2) \cdot |\buffer^{\ell}|$. 

Failure of an iteration can happen in any of the $\ceil{\frac{1}{\epsilon}} + O(\log(k))$ steps of the two for loops. For each step $1 \leq z \leq \ceil{\frac{1}{\epsilon}} + O(\log(k))$, we denote the probability that the current iteration is terminated at step $z$ at a failed state with $P_z$. The probability that an iteration will not fail can then be written as 
\[
\prod_{z} (1-P_z).
\]
In the rest of the proof, we show that this quantity is at least a constant for any constant $\epsilon > 0$.

First, we show that 
at any of the  $\ceil{\frac{1}{\epsilon}}$ steps of the first for loop, the probability of failing is less than $\epsilon/2$. 
Let us consider step $1 \leq z \leq \ceil{\frac{1}{\epsilon}}$. 
We focus on the beginning of step $z$ and upper bound $P_z$ for any possible outcome of the previous steps $1, \cdots, z-1$. 
Let $S$ be the set of selected elements in all the first $z-1$ steps.
 If at least $\epsilon/2$ fraction of elements in $\buffer^{\ell}$ has a marginal value less than $\tau$ to $S$, we can say that the iteration will not fail for the rest of the steps for sure (with probability $1$). 
 We note that as $S$ grows the marginal values of elements to it will not increase, so at least $\epsilon/2$ fraction of elements will be filtered out independent of which step the process terminates.

So we focus on the case that less than $\epsilon/2$ fraction of elements 
in $\buffer^{\ell}$ have marginal value less than $\tau$ to $S$. 
Since, in the first loop, we pick one of them randomly and look at its marginal value as a test to whether terminate the iteration or not, the probability of termination at this step $z$ is not more than $\epsilon/2$ and therefore $P_z$ is also at most $\epsilon/2$.

In the second for loop, at Lines~\ref{line:loop2:begin}--\ref{line:loop2:end}, we have a logarithmic number of steps and we can upper bound the probability of terminating the iteration in a failed state at any of these steps in a similar way. The main difference is that instead of sampling one random element from $\buffer$, we pick $t$ random elements and look at their average marginal value together as a test to whether terminate the current iteration or not. 

We want to upper bound the probability of terminating the iteration in a step $z > \ceil{\frac{1}{\epsilon}}$ at a failed state. 
This will happen if at the step $z$ the \AlgSampling algorithm picks a random subset $T$ with
\begin{itemize}
\item  $\dfrac{f(T \mid S)}{|T|} \leq (1-\epsilon)\tau$, and
\item  also less than $\epsilon/2$ fraction of elements in $\buffer^{\ell}$ has a marginal value less than $\tau$ to $T \cup S$. 
\end{itemize}

We look at the process of sampling $T$ as a sequential process in which we pick $t$ random elements one by one. 
We can call each of these $t$ parts a small random experiment. 
We note that the first property above holds only if in at least $\epsilon t$ of these smaller random experiments the marginal value of the selected element to the current set $S$ is below $\tau$. 
We also assume that we add the selected elements to $S$ as we move on. 
We simulate this random process with a binomial process of tossing $t$ independent coins.
 If the marginal value of the $i$-th sampled element to $S$ is at least $\tau$, we say that the associated coin toss is a head. 
 Otherwise, we call it a tail. The probability of a tail depends on the fraction of elements in $\buffer^{\ell}$ with marginal value less than $\tau$ to $S$. 
 If this fraction at any point is at least $\epsilon/2$, we know that the second necessary property for a failed iteration does not hold anymore and will not hold for the rest of the steps. 
 Therefore the failure happens only if we face at least  $\epsilon t$ tails each with probability at most $\epsilon/2$. The rest of the analysis is applying simple concentration bounds for different values of $t$.

So we have a binomial distribution with $t$ trials each with head probability at least $1-\epsilon/2$, and we want to upper bound the probability that we get at least $\epsilon t$ tails. The expected number of tails is not more than $\epsilon t /2$ so using Markov's inequality, the probability of seeing at least $\epsilon t$  tails is at most $0.5$. Furthermore, for larger values of $t$ we can have much better concentration bounds. 

Using Chernoff type bounds in Lemma~\ref{lem:chernoff}, we know the probability of observing at least  $\epsilon t$ tails is not more than: 

\[
P_z \leq  \Prob{}{ \textrm{tails} - \epsilon t/2 \ge \epsilon t /2} \le e^{-\epsilon t / 10}.
\]

As we proceed in steps, the number of samples $t$ grows geometrically.
Consequently, the failure probability declines exponentially (double exponential in the limit). 

So the number of steps it takes to reach the failure probability declining phase is a function of $\epsilon$ and therefore it is a constant number.
We conclude that for any constant $\epsilon > 0$, the probability of not failing in an iteration, i.e., $\prod_z (1-P_z)$, is lower bounded by a constant  $\zeta_{\epsilon} > 0$.
 Since any iteration will terminate eventually, we can say that for any iteration with constant probability an $\epsilon/2$ fraction of elements will be filtered out of $\buffer$. 
So the expected size of $\buffer$ after $X$ iterations will be at most $2^{-\Omega(X)} n$ where $n$ is the number of input elements at the beginning of \AlgSampling. 
So the probability of having more than $C \log(|\buffer|)$ iterations decreases exponentially with $C$ for any coefficient $C$ using Markov's inequality which means the expected number of iterations is $O(\log(|\buffer|))$.
\end{proof}

\begin{lemma}[Chernoff bounds, \cite{bansal2006santa}]
\label{lem:chernoff}
Suppose $X_1,\dots,X_n$ are binary random variables such that
  $\Prob{}{X_{i}=1} = p_i$. Let $\mu = \sum_{i=1}^n p_i$ and
$X = \sum_{i=1}^n X_i$. Then for any $a > 0$, we have
\[
  \Prob{}{X - \mu \ge a} \le e^{-a \min\parens*{\frac{1}{5}, \frac{a}{4\mu}}}.
\]
Moreover, for any $a > 0$, we have
\[
  \Prob{}{X - \mu \le - a} \le e^{-\frac{a^2}{2\mu}}.
\]
\end{lemma}

\section{Twitter Dataset Details} \label{sec:twitter-function}

\subsection{Intuition}
To clean the data, we removed punctuation and common English words (known as stop words, thus leaving each individual tweet as a list of keywords with a particular timestamp. To give additional value to more popular posts, we also saved the number of retweets each post received. 

Therefore, any individual tweet $t$ consists of a set of keywords $K_t$ and a value $v_t$ that is the number of retweets divided by the number of words in the post. 

A set of tweets $T$ can be thought of as a list of $(keyword,score)$ pairs. The keywords $K_T$ in a set $T$ is simply a union of the keywords of the tweets in T:

\[
K_T = \bigcup_{t \in T} K_{t}
\]

The score $s_k$ of each keyword $k \in K_T$ is simply the sum of the values of posts containing that keyword. That is, if $T_k \subseteq T$ is the subset of tweets in $T$ containing the keyword $k$, then:
\[
s_k = \sum_{t \in T_k} v_t
\]
Therefore, we define our submodular function $f$ as follows:
\[
f(T) = \sum_{k \in K_T} \sqrt{s_k}
\]

Intuitively, we sum over all the keyword scores because we want our set of tweets to cover as many high-value keywords as possible. However, we also use the square root to introduce a notion of diminishing returns because once a keyword already has a high score, we would prefer to diversify instead of further picking similar tweets. 

\subsection{General Formalization}

In this section, we first rigorously define the function used for Twitter stream summarization in \cref{sec:twitter}. We then prove this function is non-negative and monotone submodular.

\paragraph{Function Definition}
Consider a function $f$ defined over a ground set $\ground$ of items.
Each item $e \in \ground$ consists of a positive value $\text{val}_e$ and a set of $\ell_e$ keywords $W_e = \{ w_{e,1}, \cdots, w_{e, \ell_e}\}$ from a general set of keywords $\cW$.
The score of a word $w \in W_e $ for an item $e$ is defined by $\text{score}(w,e) = \text{val}_e$.
If $w \notin W_e $, we  define $\text{score}(w,e) = 0$.
For a set $S \subseteq V$ the function $f$ is defined as follows:
\begin{align} \label{eq:function-twitter}
	f(S) = \sum_{w \in \cW} \sqrt{\sum_{e \in S} \text{score}(w,e)}.
\end{align}

\begin{lemma}
	The function $f$ defined in Eq.~\eqref{eq:function-twitter} is non-negative and monotone submodular.
\end{lemma}
\begin{proof}
	The not-negativity and monotonicity of $f$ are trivial.
	For two sets $A \subset B$ and $e \in \ground \setminus B$ we show that
	\[f(\{e\}  \cup A  ) - f(A) \geq f(\{e\}  \cup B) - f(B).\]
	To prove the above inequality, assume $W_e = \{ w_{e,1}, \cdots, w_{e, \ell_e}\}$  is the set of keywords of $e$.
	For a keyword $w_{e,i}$ define $a_{w_{e,i}} = \sum_{u \in A} \text{score}(w_{e,i},u)$
	and $b_{w_{e,i}} = \sum_{u \in B} \text{score}(w_{e,i},u)$.
	It is obvious that $a_{w_{e,i}} \leq b_{w_{e,i}}$.
	It is straightforward to show that
	\[\textstyle \sqrt{a_{w_{e,i}} + \text{score}(w_{e,i},e)}  - \sqrt{a_{w_{e,i}}  } \geq \sqrt{b_{w_{e,i}} + \text{score}(w_{e,i},u)} - \sqrt{b_{w_{e,i}} }. \]
	If sum over all keywords in $W_e$ then the submodularity of $f$ is proven.
\end{proof}

\section{More Experimental Results} \label{more_graphs}

In this section, we will present a few more graphs that we didn't have space for in the main paper.

\subsection{Single-Source Experiments}

\begin{figure*}[tb!] 
	\centering     
\subfloat[]{\includegraphics[height=50mm]{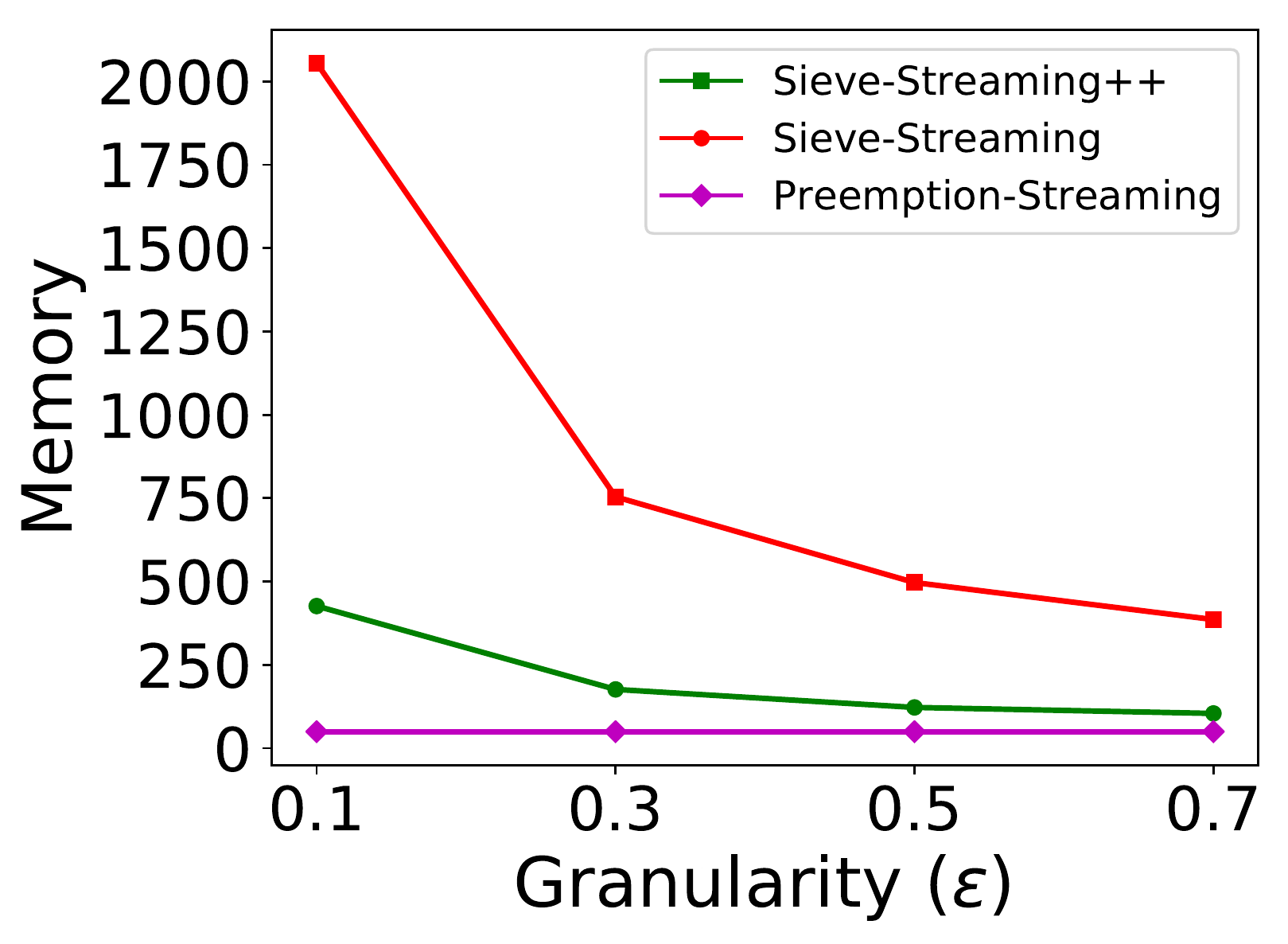}}
\qquad
\subfloat[]{\includegraphics[height=50mm]{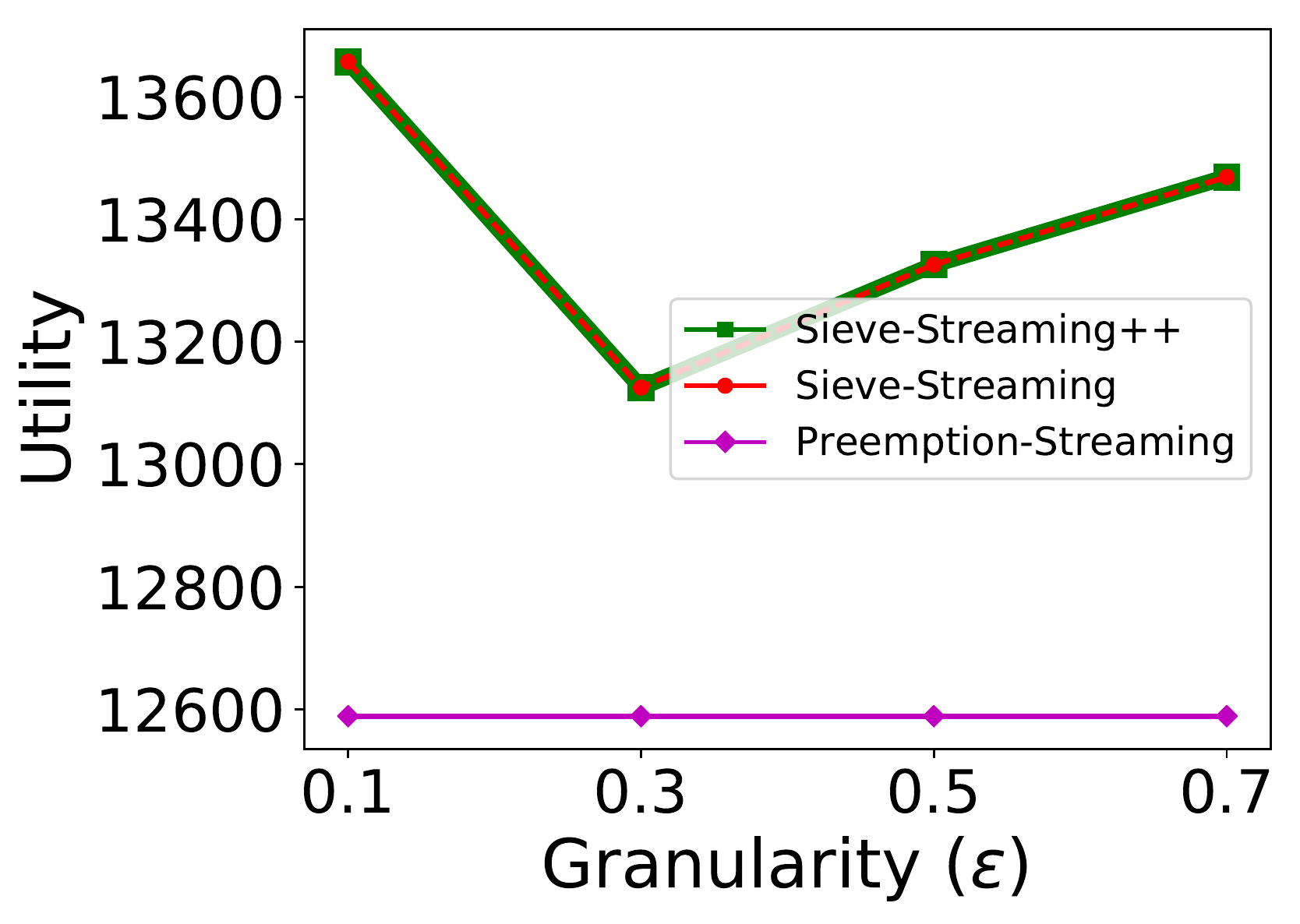}}\\
\subfloat[]{\includegraphics[height=50mm]{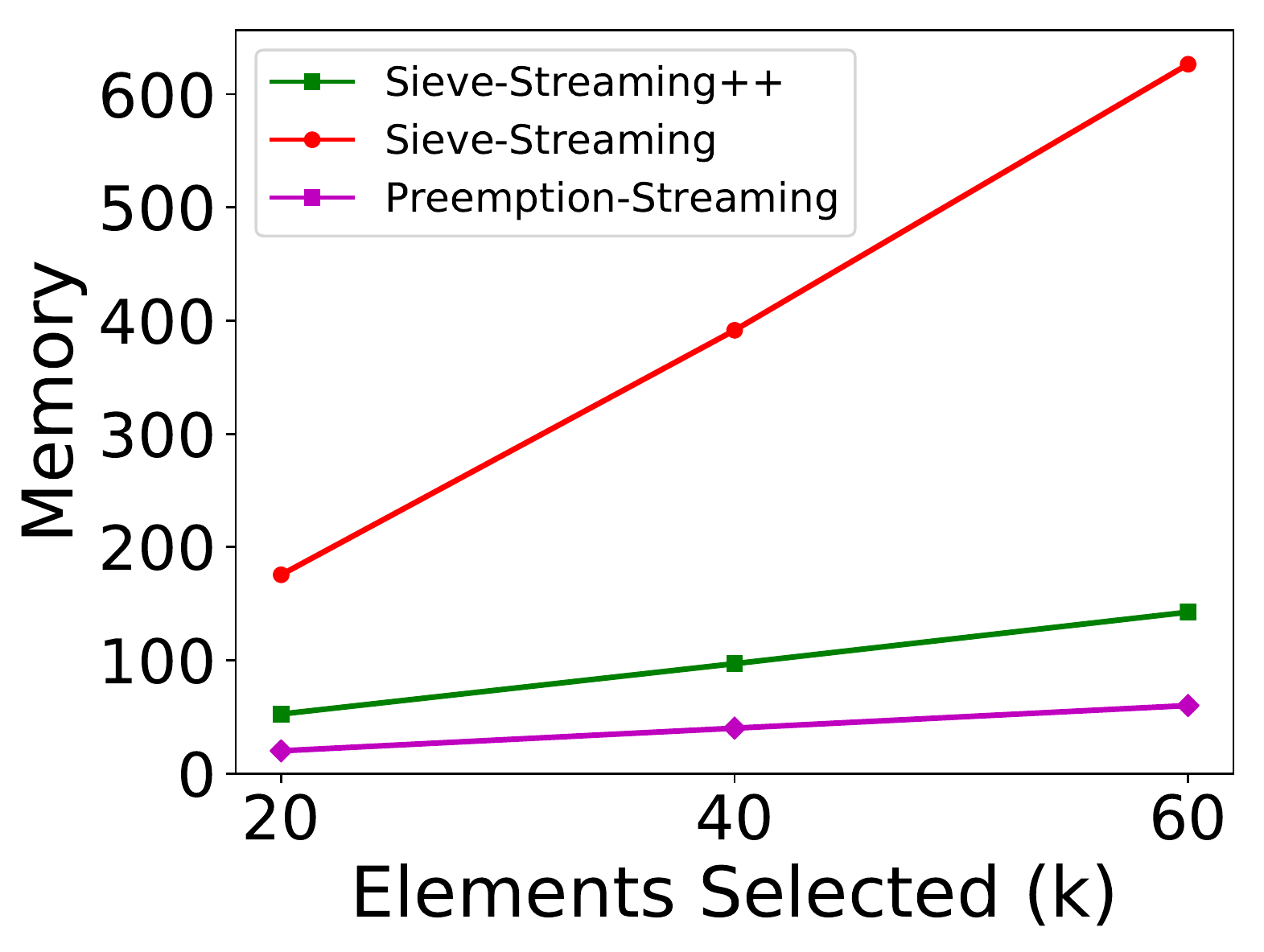}}
\qquad
\subfloat[]{\includegraphics[height=50mm]{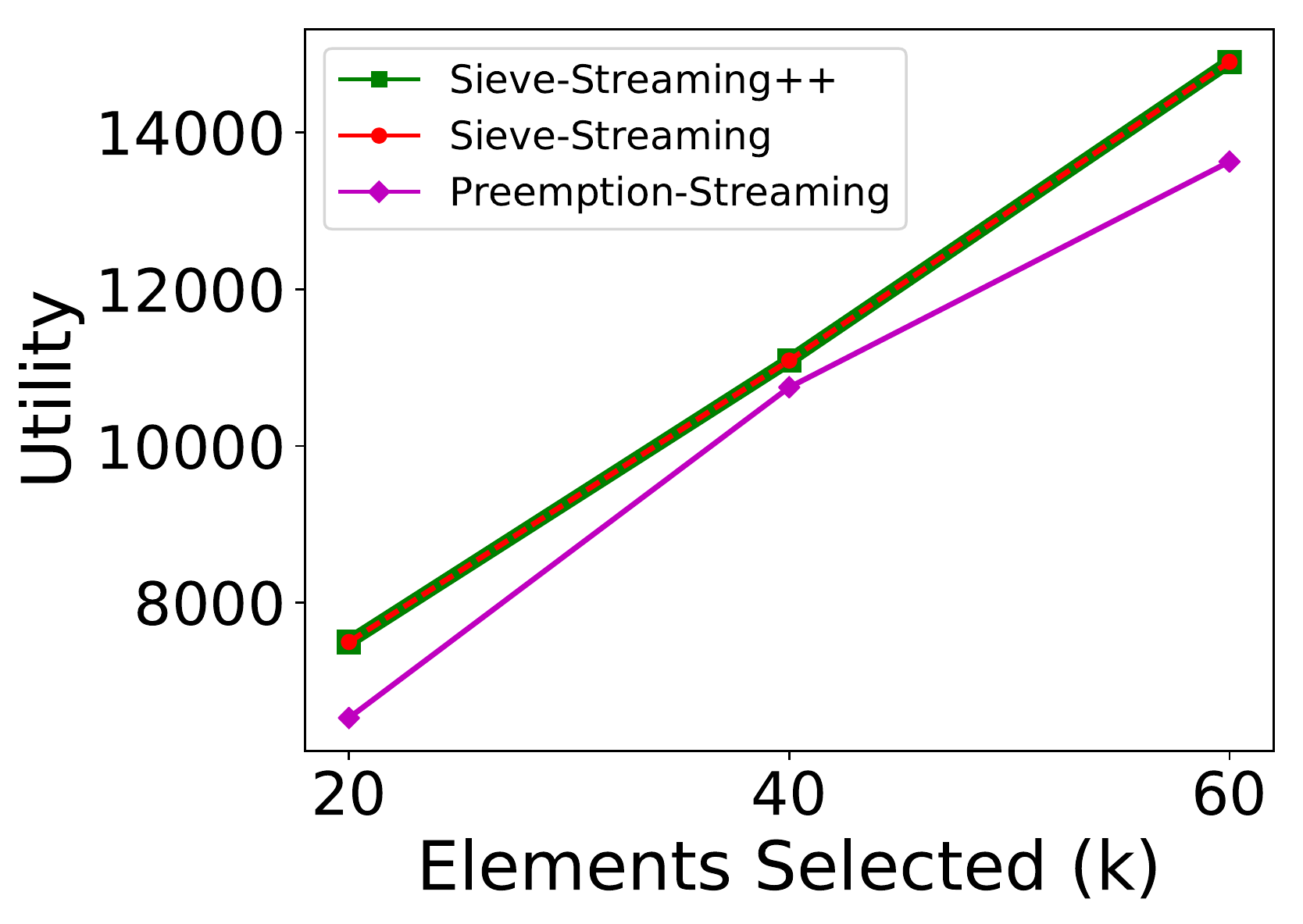}} 
\caption{Single-source streaming results on the Twitter summarizaton task. In (a) and (b), $\epsilon = 0.5$. In (c) and (d), $k = 50$.}
\label{appendixGraphs1}
\end{figure*}

Here we present the set of graphs that we displayed in Figure \ref{sieveGraphs}, except here they are run on the Twitter dataset instead. For the most part, they are showing the same trends we saw before.  \AlgSieveStreamingPlus has the exact same utility as  \AlgSieveStreaming, which is better than \AlgStreamingPreemption. We also see the memory requirement of \AlgSieveStreamingPlus is much lower than that of \AlgSieveStreaming, as we had hoped. 

The only real difference is in the shape of the utility curve as $\epsilon$ varies. In Figure \ref{sieveGraphs}, the utility was decreasing as $\epsilon$ increased, which is not necessarily the case here. However, this is relatively standard because changing $\epsilon$ completely changes the set of thresholds kept by \AlgSieveStreamingPlus, so although it usually helps the utility, it is not necessarily guaranteed to do so.

Also, note that we only went up to $k = 60$ in this experiment because \AlgStreamingPreemption was prohibitively slow for larger $k$.

\subsection{Multi-Source Experiments}

\begin{figure*}[htb!] 
	\centering     
\subfloat[]{\includegraphics[height=47mm]{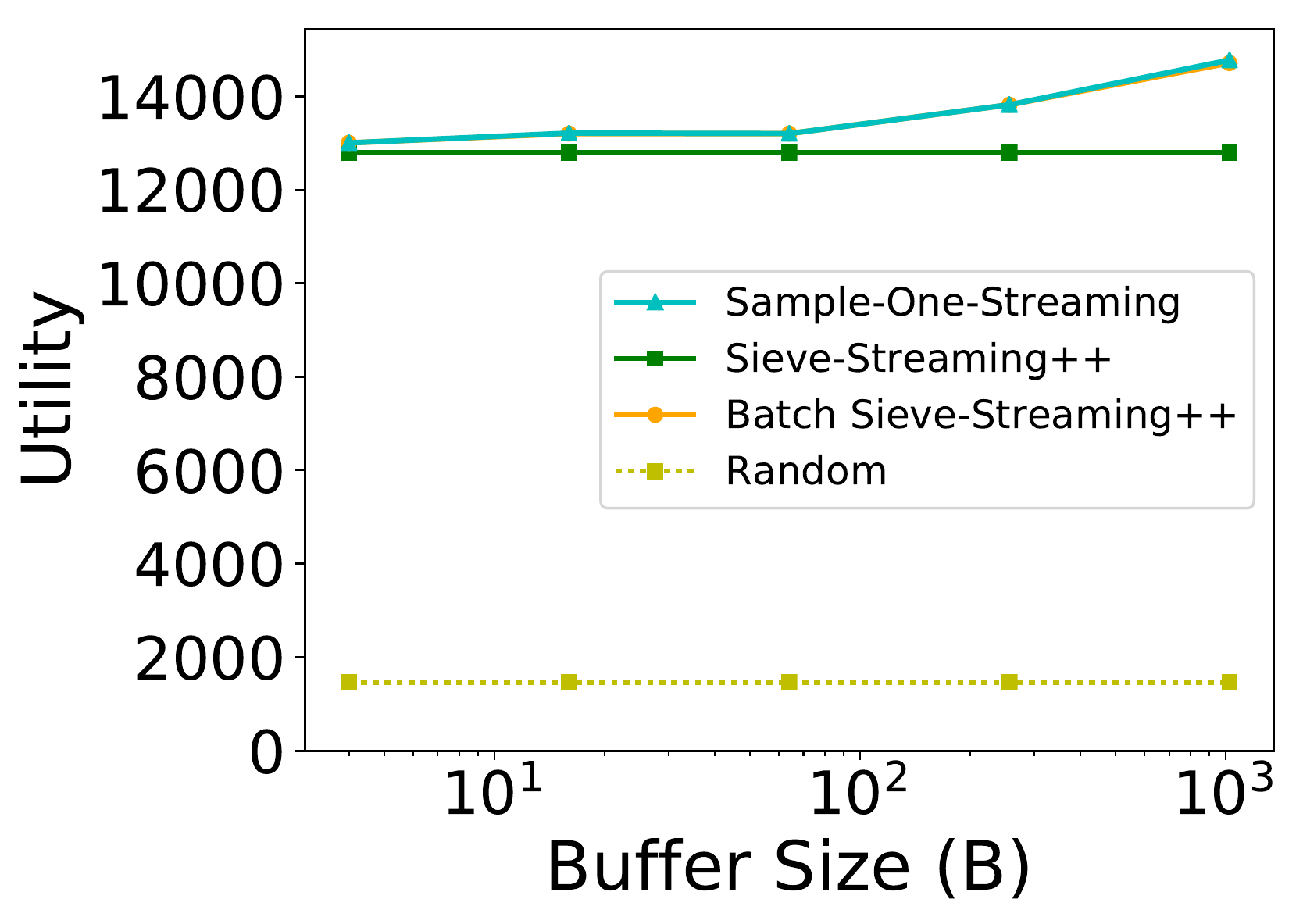}\label{twitterUtil2}} 
\qquad
\subfloat[]{\includegraphics[height=47mm]{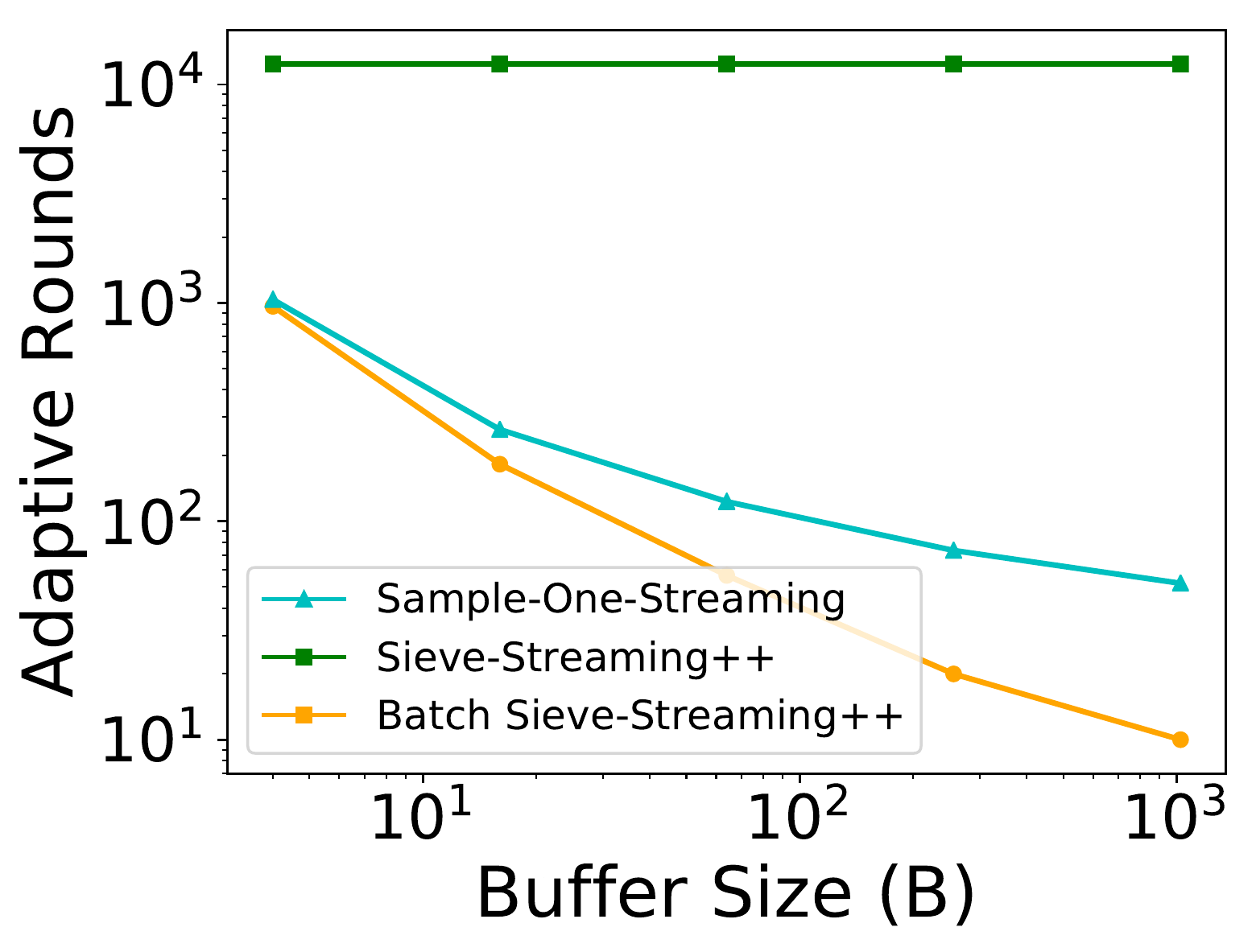}\label{twitterAdapt2}}\\
\subfloat[]{\includegraphics[height=47mm]{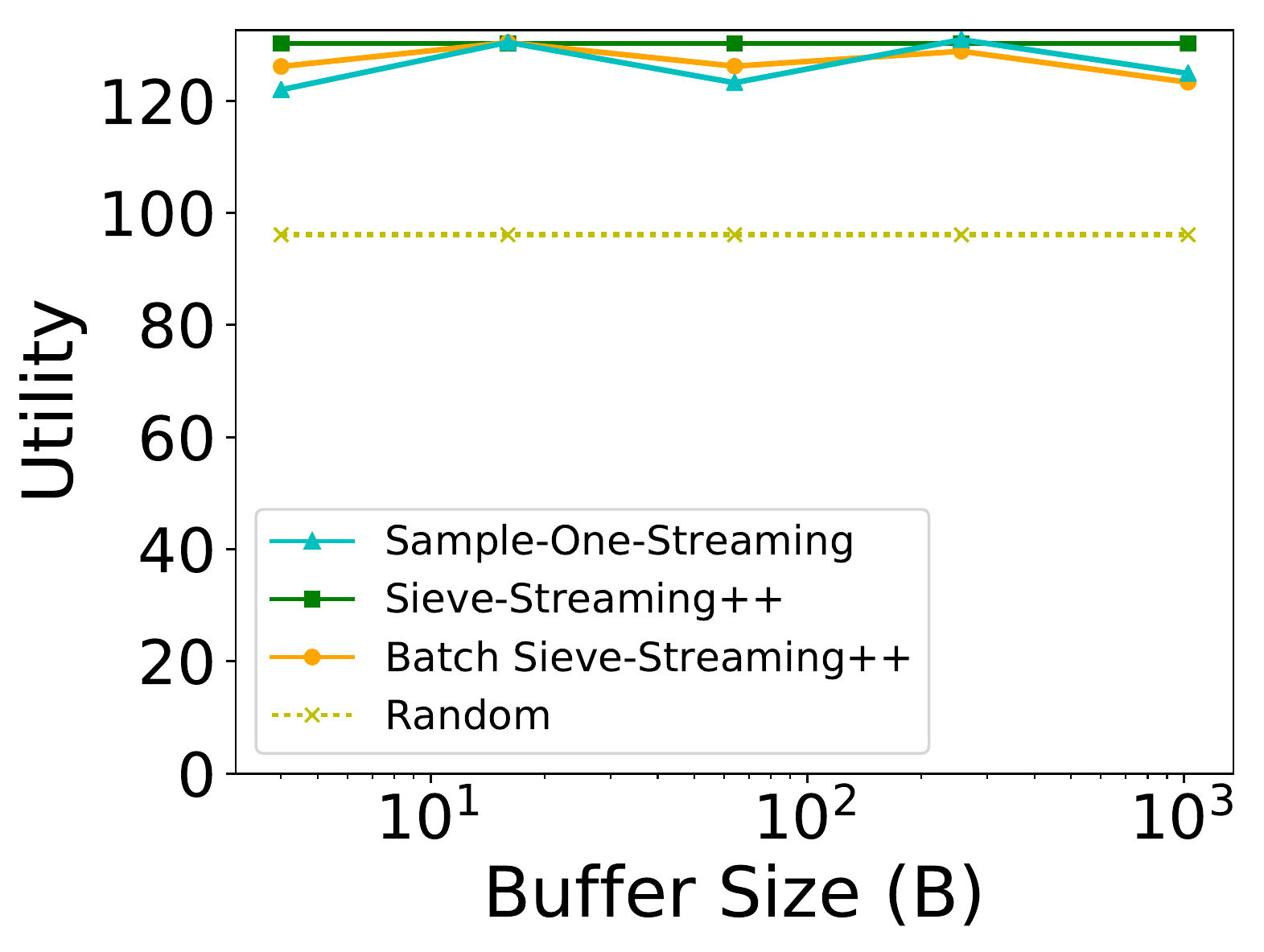}\label{startYou}}
\qquad
\subfloat[]{\includegraphics[height=47mm]{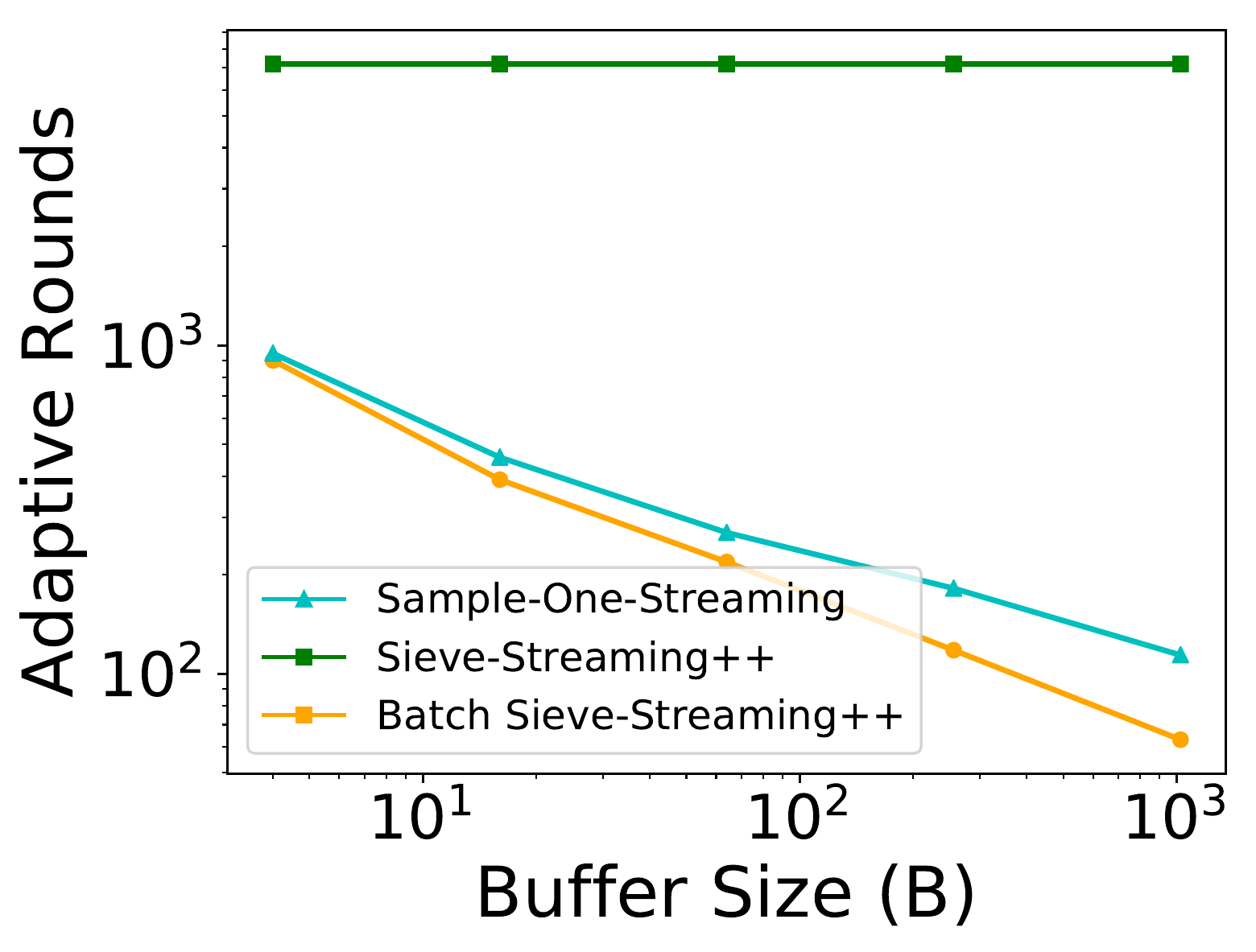}}
\\
\subfloat[]{\includegraphics[height=47mm]{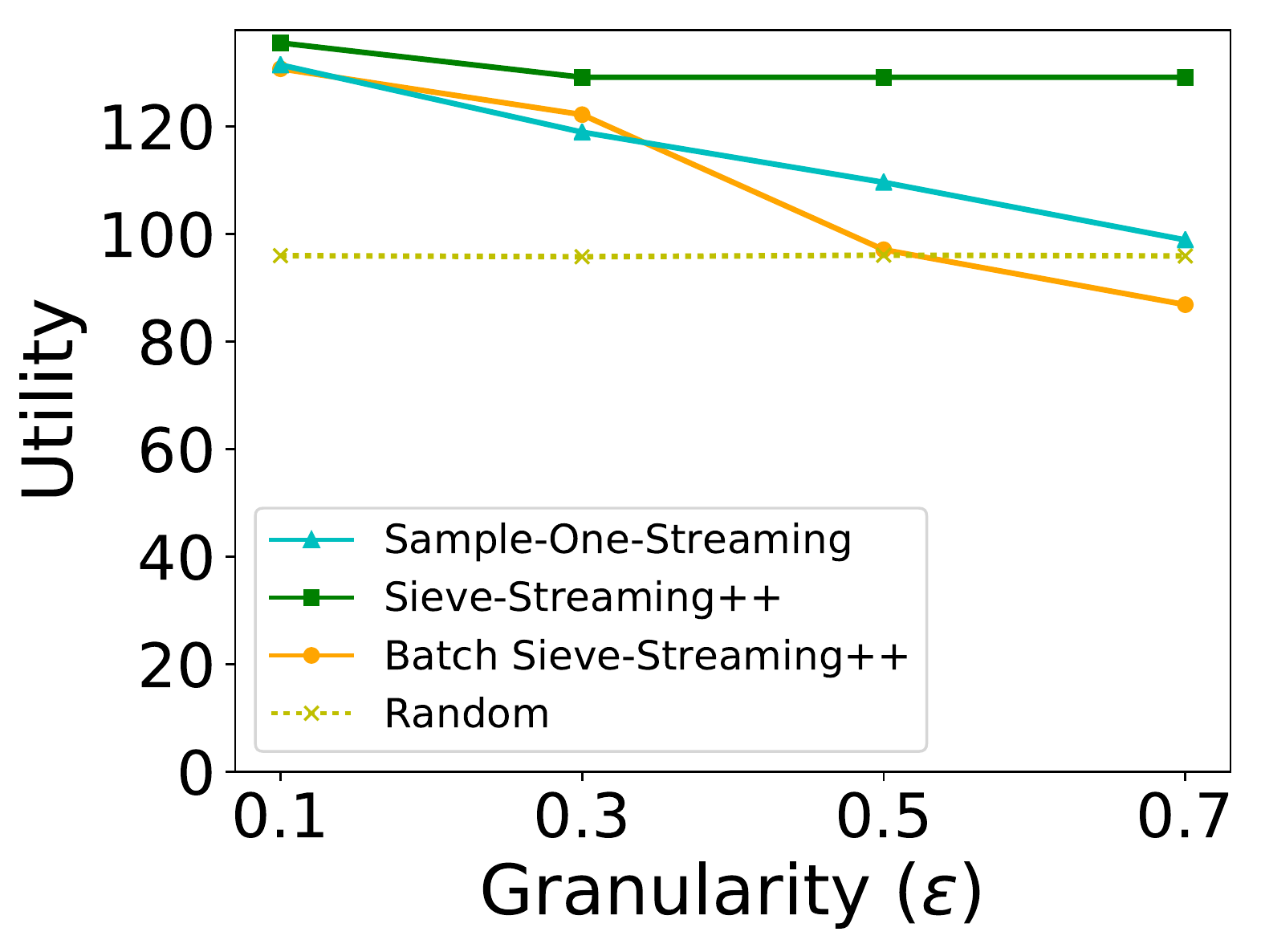}}
\qquad
\subfloat[]{\includegraphics[height=47mm]{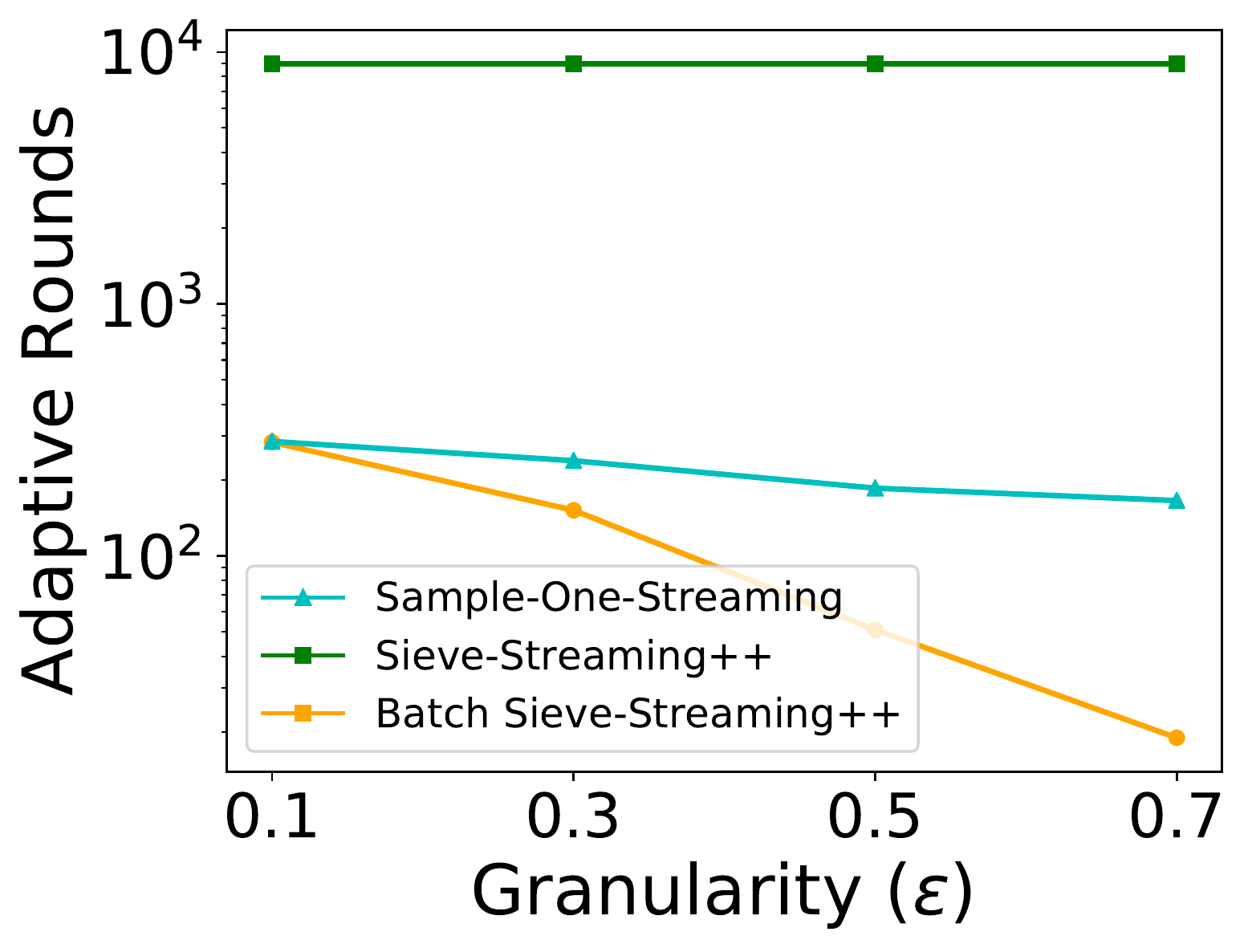}\label{endYou}} 
\caption{Additional multi-source graphs. (a) and (b) are additional graphs for the Twitter dataset, this time with $\epsilon = 0.6$ and $k = 50$. (c) through (f) are the equivalent of Figure \ref{hybridGraphs}, but for the YouTube dataset. Unless they are being varied on the x-axis, we set $\epsilon = 0.25$, $\memory = 100$, and $k = 100$.}
\label{appendixGraphs2}
\end{figure*}

In Figure \ref{twitterUtil}, \AlgSieveStreamingPlus had the best utility performance. In Figure \ref{twitterUtil2}, we set $k = 50$ and $\epsilon = 0.6$ and now we see that \AlgHybrid and \AlgOne have higher utility, and that this utility increases as the buffer size increases. However, in this case too, the main message is that the utilities of the three algorithms are comparable, but \AlgHybrid uses the fewest adaptive rounds (Figure \ref{twitterAdapt}).

In Figures \ref{startYou} through \ref{endYou}, we display the same set of graphs as Figure \ref{hybridGraphs}, but for the YouTube experiment. In the YouTube experiment, it is more difficult to select a set of items that is significantly better than random, so we need to use a smaller value of $\epsilon$. We see that for smaller $\epsilon$, the difference in adaptive rounds between \AlgHybrid and \AlgOne is smaller. This is consistent with our results because the number of adaptive rounds required by \AlgOne does not change much with $\epsilon$, while the number of adaptive rounds of \AlgHybrid increases as $\epsilon$ gets smaller.

\end{document}